\documentclass{article}
\usepackage{preprint,times}
\usepackage{amsmath,amsfonts,bm}

\def\eqref#1{equation~\ref{#1}}
\def\1{\bm{1}}

\DeclareMathAlphabet{\mathsfit}{\encodingdefault}{\sfdefault}{m}{sl}
\SetMathAlphabet{\mathsfit}{bold}{\encodingdefault}{\sfdefault}{bx}{n}

\usepackage{hyperref}
\usepackage{url}
\usepackage{amsmath,amssymb,amsthm}
\usepackage{booktabs}
\usepackage{graphicx}
\usepackage{placeins}
\usepackage{xcolor}
\usepackage{tikz}
\usetikzlibrary{calc}

\definecolor{nullamber}{HTML}{B26A00}

\newtheorem{theorem}{Theorem}
\newtheorem{proposition}{Proposition}
\newtheorem{corollary}{Corollary}
\newtheorem{definition}{Definition}
\newtheorem{lemma}{Lemma}

\title{The Lattice of Transition Laws}

\author{T. Y. Tsui, Jiatao Gu, Lingjie Liu \\
University of Pennsylvania}

\begin{document}

\maketitle

\begin{abstract}
Diffusion and autoregression (AR) have long been seen as different categories of generative models, with diffusion specialising in continuous fields and AR specialising in discrete tokens. Recent work seeks to combine the advantages of the two models, and each hybrid fixes its decoding schedule by design. In this paper, we ask whether the performance of decoding schedules of one model can be predicted before decoding at a fixed number of steps. We describe diffusion, AR, and models in between as paths on one corruption lattice, and define the cost of a schedule as the dependence its parallel steps discard. The cost shows that the fewest steps of a zero-cost schedule are set by the geometry of the data, in the same way for tokens and for continuous fields. In particular, for data that are Markov on a graph and dependent along its paths, the fewest steps equal the graph's treedepth, which is logarithmic in the length of a sequence and linear in the side length of a grid. With fewer steps than the treedepth, every schedule pays a positive cost, whose ranking we predict before decoding with a kernel of pairwise dependence estimated from pretrained weights. Across text generation, image generation, and video generation, we verify most of the predictions about the rankings of different schedules under different metrics and benchmarks. This work therefore provides a design principle for decoding for future AR models, diffusion models, and anything in between. Our code is available at https://github.com/TSUITUENYUE/The-Lattice-of-Transition-Laws.
\end{abstract}

\section{Introduction}

Autoregression, the default method for discrete tokens, generates
each coordinate conditioned on what is already revealed
\citep{vandenoord2016pixelrnn}. Diffusion, the default for continuous
fields, progressively removes noise from all coordinates together
\citep{ho2020ddpm, song2021sde}. Recent work combines their advantages in both domains. For continuous fields, coordinate-by-coordinate generation conditions each coordinate on the content
generated before it \citep{li2024mar} and lets a video grow to any length
\citep{chen2024diffusionforcing, skyreels2025}. For discrete tokens, parallel updates reduce the
one step per token required by autoregression \citep{nie2025llada}.

In existing hybrid models, the design of each model fixes how autoregression and diffusion are combined \citep{arriola2025bd3lm, tidar2025, sahoo2025esolm, wu2023ardiffusion, ruhe2024rolling, kim2024fifo}. However, their measured quality reports only the schedule each model fixes. The schedule matters: with the same weights and the same few-step budget, the image quality of MAR varies by an order of magnitude across reveal orders (Table~\ref{tab:mar}). Yet we lack a design principle that decides the schedule for a given model and data. A candidate schedule can only be judged after decoding with it. We therefore ask whether the performance of the decoding schedules of one model can be predicted before decoding.

Such a principle first needs a space that contains every schedule it can choose from and compare against. We assign a corruption level to each coordinate and let the levels vary independently, so that some coordinates can be clean while others are partly or fully corrupted. These assignments form the corruption lattice $L^d$. A decoding schedule then decides a path through the lattice: at each step, it chooses which coordinates to update and by how much (Figure~\ref{fig:hero}). 

A unifying space alone is not sufficient to be a design principle. Ranking different schedules needs a quantified cost. Earlier work also unifies the two families on discrete tokens: the training objective of masked diffusion is that of an any-order autoregressive model \citep{ou2025absorbing}, and a per-position noise schedule on tokens expresses autoregression as a special case of diffusion \citep{fathi2025hyperschedules}. Our lattice holds continuous fields as well, where Diffusion Forcing trains video with an independent noise level per frame, as next-token prediction with uncertainty \citep{chen2024diffusionforcing}. These prior works describe schedules, while our lattice adds a cost that compares them before decoding and predicts their performance. On our lattice, every step has an exact joint law for the coordinates it updates, and the cost of a parallel step is the dependence lost when each of them is drawn from its own conditional, which adds up along the path to the schedule's divergence from exact sampling (Section~\ref{sec:design}). For data that are Markov on a graph and dependent along its paths, the ideal zero-cost schedule follows an elimination tree of the graph, updating the coordinates at one depth of the tree per step. The least depth of such a tree, the treedepth, is a lower bound on the number of steps of any zero-cost schedule (Section~\ref{sec:optimal}).

With even fewer steps than the treedepth, every schedule will need to pay a cost. Two coordinates depend on each other less the farther apart they are, and a step that spreads its coordinates apart pays less than one that updates neighbours. Before decoding, we measure from a pretrained model how this dependence decreases with distance, as a kernel over pairs of positions, and use it to predict how the schedules of the model rank (Section~\ref{sec:choose}).

\begin{figure}[t]
\centering
\includegraphics[width=0.9\textwidth]{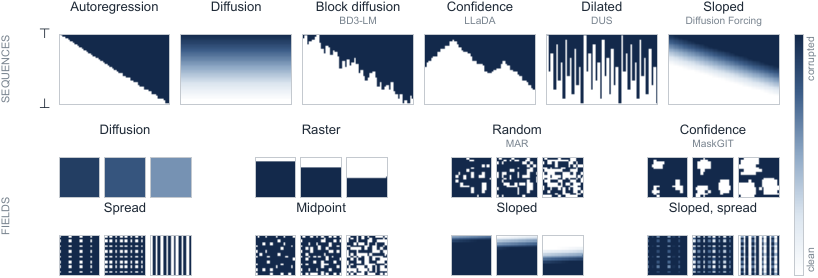}
\caption{Decoder schedules on the lattice. Top: columns are coordinates, rows are successive
states from $\top$ to $\bot$, and shades run from white (clean) to blue (corrupted).
Bottom: image-token schedules at 12.5\%, 25\% and 50\% of the path; the last two show noise gradients in
raster and spread order; Table~\ref{tab:rules} defines the rules.}
\label{fig:hero}
\end{figure}

In this paper, our contributions are:
\begin{itemize}\setlength{\itemsep}{1pt}
\item We place diffusion, autoregression and every schedule between them on one corruption lattice, and show that one set of weights can traverse it.
\item We define the dependence cost of a schedule, the dependence its parallel steps discard, show that it equals the divergence of the path from exact sampling, for both discrete tokens and continuous values, and derive that for data Markov on a graph and dependent along its paths the fewest zero-cost steps equal the graph's treedepth.
\item We predict how schedules rank before decoding from a kernel of pairwise dependence estimated from pretrained weights, and test the predictions on released language, image and video models.
\end{itemize}

Related work is discussed in Appendix~\ref{sec:related}.

\section{Decoding Schedules on the Corruption Lattice}
The space must hold every decoder as a sequence of transition laws, each conditioned on a partially corrupted state. We write diffusion and autoregression in this form first. The forward process adds Gaussian noise with
variance $a$, $y_a = x_0 + \sqrt{a}\,\varepsilon$, at levels $a_0 > \dots > a_B$. A transition between two adjacent levels requires the conditional mean
$\hat{x}_k(y) = \mathbb{E}[x_0 \mid y_{a_k} = y]$.  Sampling fresh noise around the exact conditional mean of $y_{k+1}$ gives the Gaussian law
\begin{equation}
\hat p(y_{k+1} \mid y_k) = \mathcal{N}\!\big(r\,y_k + (1-r)\,\hat{x}_k(y_k),\;
\tau_k I\big), \qquad \tau_k = a_{k+1}(a_k - a_{k+1})/a_k,
\label{eq:stochlaw}
\end{equation}
with $r = a_{k+1}/a_k$. Carrying the noise of $y_k$ forward gives
the deterministic law
\begin{equation}
y_{k+1} = \hat{x}_k + \sqrt{r}\,\big(y_k - \hat{x}_k\big).
\label{eq:diraclaw}
\end{equation}
Both diffusion laws use $\hat{x}_k$, with $a_k$ identifying the noise level of the conditioning state. Autoregression likewise specifies a transition law conditioned on a
partially corrupted state. In the next-token conditional
$p(x_i \mid x_{<i})$, generated coordinates are clean and the rest are fully corrupted, and each step takes one coordinate from fully corrupted to clean. The causal mask carries these level labels.

\subsection{One level per coordinate}

Let coordinate $i$ carry its own level $\ell_i \in L$, ordered from clean to fully corrupted: a noise variance in $[0, \infty]$ for continuous data, and for tokens the absorbing channel, whose $\ell \to \infty$ end is the mask. Deployed token decoders use only the two ends \citep{austin2021d3pm, sahoo2024mdlm, nie2025llada}; an interior level lets a step also choose how far to advance each coordinate (Appendix~\ref{app:graded}). Conditioning on a lattice point $\ell$ includes the corrupted values it carries, which we leave implicit.

\begin{definition}[Corruption lattice]
\label{def:lattice}
The corruption lattice is $L^d$, ordered coordinatewise. A decoding schedule is a monotone path from $\top$, every coordinate corrupted,
to $\bot$, every coordinate clean. A step from $\ell$ to
$\ell' \preceq \ell$ advances the coordinates
$S = \{i : \ell'_i < \ell_i\}$. Its transition law is
$p_{\ell\to\ell'}(x'_S \mid \ell)$, with $x'_S$ the updated
values at the destination levels $\ell'_S$. Coordinates outside $S$
retain their values.
\end{definition}
When the step or schedule fixes the destination levels, we write
$p(x_S \mid \ell)$ for the transition law, with $x_S$ the updated values. Entropies and mutual
information below use the same convention.

\begin{lemma}[Path representation]
\label{prop:paths}
Let $F$ be the family of transition laws $p_{\ell\to\ell'}$ between lattice points $\ell' \preceq \ell$. Every monotone path composes elements of $F$ into a law whose marginal on the clean coordinates is the data distribution, and two paths differ only in which elements of $F$ they evaluate.
\end{lemma}

The autoregressive chain rule is the sequential path, one coordinate at a time, and the diffusion chain the synchronised path on the noise axis; we define the lattice's corners as these two paths. A trained model supplies the laws of Lemma~\ref{prop:paths} at the lattice
points its training distribution covers, and trained diffusion and autoregressive models differ in that distribution as well as in their schedules.

\subsection{One model for every schedule}
\label{sec:corners}

The rest of the paper compares schedules on this lattice. We first show that one model can traverse it: one model, trained on a distribution over the whole lattice, decodes both corners and the schedules between them with the same weights (Figure~\ref{fig:corners}). At the sequential token corner, its bits per character (bpc) are 0.06 below those of a dedicated masked model trained on random subsets. On pixels, the schedules it decodes span an order of magnitude in step count.

\begin{figure}[t]
\centering
\begin{minipage}[c]{0.61\textwidth}
\centering
\includegraphics[width=\linewidth]{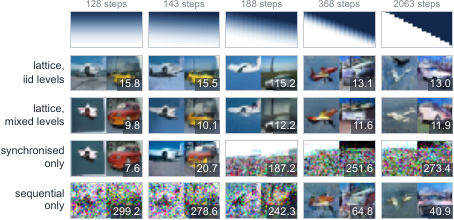}
\end{minipage}\hfill
\begin{minipage}[c]{0.37\textwidth}
\centering
\renewcommand{\arraystretch}{0.96}\setlength{\tabcolsep}{3pt}
\begin{tabular}{@{}lc@{}}
\toprule
 & bpc \\
\midrule
\multicolumn{2}{l}{AR chain} \\
\quad Dedicated causal & 1.22 \\
\quad Dedicated masked & 1.40 \\
\quad Lattice & 1.34 \\
\quad Two-stream & 1.25 \\
\midrule
\multicolumn{2}{l}{Masked ELBO} \\
\quad Dedicated masked & 1.54 \\
\quad Lattice & 1.67 \\
\quad Two-stream & 1.87 \\
\bottomrule
\end{tabular}
\end{minipage}
\caption{Both corners from one weight set. Pixels: columns give schedules and step counts (Figure~\ref{fig:hero}), rows give training distributions, and cells show 2 samples and CIFAR-10 FID-50K. Tokens: text8 bpc at 85M; AR-chain rows use the identity order and masked-ELBO rows the masked-diffusion bound (protocol in Appendix~\ref{app:cornergaps}).}
\label{fig:corners}
\end{figure}

Following a path requires an architecture that expresses its conditionals and training that covers it. Trained on one family of paths, a model collapses on the others (Figure~\ref{fig:corners}; Appendix~\ref{app:cornergaps}).

\section{The Dependence Cost of a Schedule}
\label{sec:design}

Independent sampling replaces each joint law in Lemma~\ref{prop:paths} with a product of coordinate conditionals. We compare schedules at a fixed number of steps by the dependence they discard.

\label{sec:conserved}

\begin{proposition}[Exact path cost of parallel advance]
\label{thm:pathcost}
Let $\pi = (\ell^{(0)}, \dots, \ell^{(T)})$ be a monotone schedule on
the corruption lattice with advanced sets $S_t$. Write $x_{S_t}$ for
the updated values at the destination levels $\ell^{(t+1)}_{S_t}$,
with the transition index fixed by $\pi$ and suppressed. Let $P^\pi$ be the law of the path that draws
each step from the exact joint law $p(x_{S_t} \mid \ell^{(t)})$, and
$\widetilde P^\pi$ the law of the path that draws each step from the
product $\prod_{i \in S_t} p(x_i \mid \ell^{(t)})$ of the same
coordinate conditionals. Then
\begin{equation}
D_{\mathrm{KL}}\big(P^\pi \,\|\, \widetilde P^\pi\big) \;=\;
\sum_{t=0}^{T-1} \mathbb{E}_{P^\pi}\Big[\mathrm{TC}\big(S_t \mid
\ell^{(t)}\big)\Big],
\qquad
\mathrm{TC}(S \mid \ell) = \sum_{i \in S} H(x_i \mid \ell)
- H(x_S \mid \ell),
\label{eq:tc}
\end{equation}
and the divergence between the two laws of the finished sample is at
most the same sum.
\end{proposition}

Proposition~\ref{thm:pathcost} holds on every monotone path of the lattice, for continuous values and interior levels as well as for tokens, and the identity of \citet{chen2025schedules} is its special case on the two-level absorbing lattice with every advanced set fixed. The two also answer different questions. Averaged over positions drawn at random, their identity depends only on how many tokens each step unmasks, through the mean information that a token shares with a random set of revealed tokens, and their schedules choose these step widths. Proposition~\ref{thm:pathcost} keeps the positions and levels of a given path and asks which coordinates each step should update. Two schedules with the same widths can differ widely in this cost, since neighbouring coordinates share more information than distant ones. On text8, switching from contiguous to spread selection at fixed widths lowers bpc by about 2 bits, while changing the widths at a fixed rule moves it by less than 0.3 bits (Appendix~\ref{app:families}). Along a random scan of the absorbing lattice, the sum is the one that \citet{cai2026confidence} control.

At an interior level, the cost is the total correlation of the values that the step reaches. Advancing a set to an interior level before the clean level reduces the cost by the information that the other coordinates' interior values carry about each coordinate beyond its own (Appendix~\ref{app:graded}). Two advances per coordinate double the step width at a fixed step budget.

The dependence cost in Proposition~\ref{thm:pathcost} vanishes on a
sequential path, where each step advances one coordinate. On a channel with a continuum of
levels it also vanishes along the synchronised path, as the level
increments tend to zero. On the two-level token channel the synchronised
path is a single step that reveals every coordinate and pays their full
total correlation.

When dependence decreases with distance, selecting adjacent coordinates incurs the largest cost at a given width. The contiguous rule of Table~\ref{tab:rules} makes this selection at every step.

A continuous decoder applies the conditional mean through
\eqref{eq:stochlaw} or \eqref{eq:diraclaw} at each step. Its measured
decoding error comprises the dependence cost in
Proposition~\ref{thm:pathcost}, errors in the trained conditionals,
and the one-step error of the continuous channel.

\section{Schedules with Zero Dependence Cost}
\label{sec:optimal}

Revealed coordinates separate the masked ones: on a chain, one coordinate per masked run is a conditionally independent set (Figure~\ref{fig:separator}). We define a coordinate as resolved once it is clean.

\begin{figure}[tb]
\centering
\includegraphics[width=0.86\textwidth]{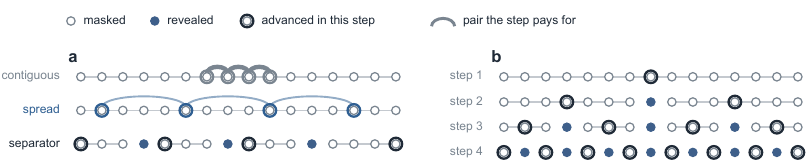}
\caption{(a) The contiguous, spread and separator rules for a width-4 step on a chain of
16 coordinates, the separator rule from a state with 3 revealed coordinates. Arcs connect consecutive advanced coordinates, with
heavier arcs indicating shorter distances. (b) The midpoint rule on
15 coordinates, advancing the midpoint of every run of masked
coordinates at each step.}
\label{fig:separator}
\end{figure}

\begin{theorem}[Separator schedules]
\label{thm:separator}
Fix a lattice state and suppose that, given its revealed coordinates, the unresolved coordinates are Markov with respect to a graph $G$, and that at this and every later state any two unresolved coordinates joined by a path of unresolved coordinates are dependent. Then the fewest steps of a schedule with zero dependence cost equal the treedepth $\mathrm{td}(G)$ of $G$.
\end{theorem}

Given the revealed coordinates, coordinates in different connected components of $G$ are independent, and a step that resolves one coordinate per component has zero cost (Appendix~\ref{app:design}). Two coordinates of one component are dependent, and a zero-cost step resolves at most one of them. Resolving a coordinate removes it from $G$ and can split its component. Hence a zero-cost schedule traces an elimination tree of $G$, in which each resolved coordinate is the parent of the components it leaves and each step resolves one depth. The least depth of such a tree is the treedepth \citep{nesetril2006treedepth}. On a chain of $d$ coordinates, the midpoint rule resolves the midpoint of every run, which halves every run, and finishes in $\lceil \log_2 (d+1) \rceil$ steps, the chain's treedepth. On an $n\times n$ grid, splitting the grid takes a whole row or column, and the treedepth lies between $n+1$ and a constant multiple of $n$ \citep{nesetril2012sparsity, bodlaender1998arboretum, george1973nested}.

The separator rule starts a step at the middle masked coordinate and adds, one at a time, the masked coordinate farthest from those already chosen, at most one per masked run until every run has one.

\section{Ranking Schedules below the Treedepth}
\label{sec:choose}

When the coordinates are dependent along the graph's paths, every schedule with fewer steps than the treedepth pays a positive cost (Theorem~\ref{thm:separator}). At a fixed step budget, we compare the candidates' path costs in \eqref{eq:tc} through a pairwise lower bound. For a specified advance from a lattice state, we define the kernel
as the conditional mutual information between two updated coordinate
values, with the destination levels suppressed,
\[
\kappa(i,j \mid \ell) \;=\; I\big(x_i;\, x_j \,\big|\, \ell\big),
\]
which is \eqref{eq:tc} for the pair $S=\{i,j\}$.

\begin{lemma}[The pairwise lower bound]
\label{prop:pairs}
For an advanced set $S = \{i_1 < \dots < i_k\}$, at any state and
under any data distribution,
\[
\mathrm{TC}(S \mid \ell) \;\ge\; \sum_{j=2}^{k}
I\big(x_{i_j};\, x_{i_{j-1}} \,\big|\, \ell\big),
\]
the kernel summed over consecutive members of $S$. Equality holds
exactly when, given the state, each selected coordinate is
independent of the earlier selected ones given its predecessor: a
first-order Markov reference.
\end{lemma}

We define a schedule's pairwise estimate by summing this bound over its steps. The estimate needs only the path and the kernel and is computed before decoding.

\begin{corollary}[The reference model]
\label{cor:island}
On the reference model of Lemma~\ref{prop:pairs}, a first-order
Markov chain at every state, a step of $k$ coordinates with consecutive
gaps at least $s$, where $\kappa \le \varepsilon$ at distances $\ge s$, pays at
most $(k-1)\,\varepsilon$. A schedule whose first step advances
$S_1 = \{i_1 < \dots < i_m\}$ and whose every later step advances the
midpoint of every run of masked coordinates pays
$\mathrm{TC}(S_1 \mid \top) = \sum_{j=2}^{m} \kappa(i_{j-1}, i_j)$ and
completes within budget $T$ when every gap of $\{0\} \cup S_1 \cup \{d+1\}$ is at most $2^{\,T-1}$.
\end{corollary}

\section{Experiments}
\label{sec:frontier}

We compare predicted and measured rankings of the selection rules of Table~\ref{tab:rules}, at matched step budgets, on our text8 model, the released MAR-B image model, which generates continuous-valued tokens with a per-token diffusion head \citep{li2024mar}, LLaDA-8B \citep{nie2025llada} and a pretrained Diffusion Forcing video model \citep{skyreels2025}. The rules keep a parallel step cheap by the distance between its coordinates or by their separation (Section~\ref{sec:optimal}), or come from released decoders.

\begin{table}[tb]
\centering\small\setlength{\tabcolsep}{3.5pt}\renewcommand{\arraystretch}{1.12}
\caption{Selection rules by principle, what each step selects, and the models they are tested on. Schematics as in Figure~\ref{fig:hero}: 16 coordinates from all masked (top) to all clean (bottom) under steps of widths 1, 2, 4 and 9, or one snapshot of an image-token grid. Contiguous, random and confidence selection follow \citet{arriola2025bd3lm}, \citet{sahoo2024mdlm} and \citet{nie2025llada}; dilated and low-discrepancy selection follow \citet{luxembourg2025dus} and \citet{besnier2025halton}.}
\label{tab:rules}
\begin{tabular}{lllcccc}
\toprule
Principle & Rule & Each step selects & Schedule & text8 & MAR & LLaDA \\
\midrule
released & contiguous & neighbours, in order (raster) & \raisebox{-0.32\height}{\includegraphics[height=0.15in]{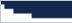}} & \checkmark & \checkmark & \\
 & random & random positions & \raisebox{-0.32\height}{\includegraphics[height=0.15in]{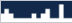}} & \checkmark & \checkmark & \checkmark \\
 & confidence & the most confident & \raisebox{-0.32\height}{\includegraphics[height=0.15in]{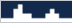}} & \checkmark & & \checkmark \\
\midrule
distance & spread & evenly spaced & \raisebox{-0.32\height}{\includegraphics[height=0.15in]{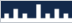}} & \checkmark & \checkmark & \checkmark \\
 & dilated & a fixed interleaved order & \raisebox{-0.32\height}{\includegraphics[height=0.15in]{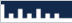}} & \checkmark & & \\
 & low-discrepancy & an evenly covering order & \raisebox{-0.32\height}{\includegraphics[height=0.17in]{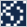}} & & \checkmark & \\[1.5pt]
 & spread, then raster & spread, then rows & \raisebox{-0.32\height}{\includegraphics[height=0.17in]{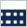}} & & \checkmark & \\
\midrule
separation & separator & one per masked run & \raisebox{-0.32\height}{\includegraphics[height=0.15in]{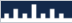}} & \checkmark & & \\
 & midpoint & farthest from revealed & \raisebox{-0.32\height}{\includegraphics[height=0.15in]{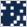}} & \checkmark & \checkmark & \checkmark \\
 & nested & middle row and column first & \raisebox{-0.32\height}{\includegraphics[height=0.17in]{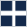}} & & \checkmark & \\
\midrule
constrained & min-distance & confident, kept apart & \raisebox{-0.32\height}{\includegraphics[height=0.15in]{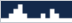}} & \checkmark & & \checkmark \\
 & separator constraint & confident, one per run & \raisebox{-0.32\height}{\includegraphics[height=0.15in]{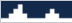}} & & & \checkmark \\
\bottomrule
\end{tabular}
\end{table}

On text8, we score the true characters after a revealed prefix against the left-to-right chain (Figure~\ref{fig:fronts}a) and against each schedule's own sequential chain, which reveals the same coordinates one at a time.

\subsection{The kernel}
\label{sec:kernel}

We estimate the two entropy terms in \eqref{eq:tc} using the trained
model. A single forward pass gives the marginal entropies
$H(x_i \mid \ell)$. To obtain the joint entropy $H(x_S \mid \ell)$,
we sum the conditional entropies along a chain that reveals the
coordinates of $S$ one at a time, requiring one forward pass per
coordinate. On text8, we estimate the kernel as a function of distance and mask rate, for pairs within a word and pairs across a word boundary (Figure~\ref{fig:kernel}a, Table~\ref{tab:kernelrates}), and evaluate it at each step's distances and mask rate. At distances up to 4, dependence within a
word is 2 to 4 times that across a boundary at the same distance;
by distance 16, the kernel is near zero.

\begin{table}[tb]
\centering\small
\caption{MAR-B kernels at revealed fraction 0.5 (left) and 8-step floors (right). The kernels give the relative displacement of a masked token's prediction when a revealed neighbour at Chebyshev distance $d$ is perturbed, and the Kullback--Leibler divergence in bits of its head law when such a neighbour is revealed. A floor sums the kernel at fraction 0.5 over consecutive tokens of each step and over the steps; spread-then-raster's floors came after its decode. FID-50K as in Table~\ref{tab:mar}.}
\label{tab:marinfo}
\begin{tabular}{ccc@{\hspace{1.4em}}lccc}
\toprule
 & \multicolumn{2}{c}{kernel} & & \multicolumn{2}{c}{floor} & \\
\cmidrule(lr){2-3}\cmidrule(lr){5-6}
$d$ & displacement & information & 8 steps & displacement & information & FID \\
\midrule
1 & 0.330 & 1.449 & raster & 77.1 & 338 & 139.80 \\
2 & 0.119 & 0.209 & spread, then raster & 46.4 & 182 & 17.44 \\
3 & 0.051 & 0.048 & random & 8.11 & 14.7 & 13.02 \\
4 & 0.030 & 0.017 & midpoint & 7.56 & 12.5 & 10.63 \\
5 & 0.021 & 0.009 & low-discrepancy & 3.87 & 1.57 & 8.32 \\
 & & & nested & 3.76 & 2.10 & 13.73 \\
 & & & spread & 3.27 & 0.641 & 9.44 \\
\bottomrule
\end{tabular}
\end{table}

\begin{figure}[t]
\centering
\includegraphics[width=0.9\textwidth]{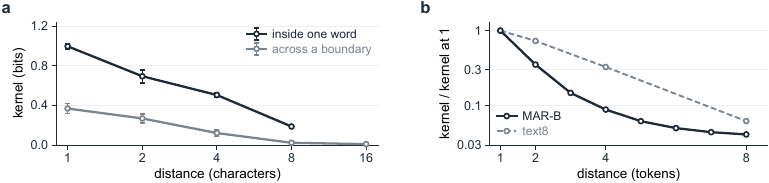}
\caption{Kernels measured before decoding. (a) Entropy decrease in bits when a second text8
coordinate is revealed, at mask rate 0.5, by sequence distance: within-word and across-boundary
pairs, mean and s.e.m. over 3 seeds, 512 pairs per cell. (b) Relative displacement of a masked-token
prediction after a revealed-token perturbation on MAR-B, at revealed fraction 0.5, by grid distance,
beside the across-boundary curve of (a); both are normalised at distance 1.}
\label{fig:kernel}
\end{figure}

On the MAR token grid, we measure dependence by the
relative change in a prediction when a revealed token is perturbed.
This response stands in for the mutual-information kernel and falls
below 5\% of its nearest-neighbour value by grid distance 8
(Figure~\ref{fig:kernel}b). An information kernel, the Kullback--Leibler divergence of a masked token's head law when one more neighbour is revealed, falls much faster with distance than this response; floors built from the two order the rules alike, up to 2 rules whose displacement floors nearly tie (Table~\ref{tab:marinfo}, Appendix~\ref{app:marinfo}).

\subsection{The ranking of schedules}
\label{sec:text8}

\begin{table}[tb]
\centering\small\setlength{\tabcolsep}{4.2pt}
\caption{Measured quality by selection rule and number of decoding steps. Left: text8 bpc under the protocol of Figure~\ref{fig:fronts}; the contiguous rule uses one step per block. Right: FID-50K on ImageNet-256 from released MAR-B weights and MAR's released evaluator, with only the reveal order changed; 1 seed per cell, except random and midpoint at 8 and 16 steps, averaged over 3 seeds within 0.25 FID. The 256-step random rule reproduces the published score.}
\label{tab:methods}\label{tab:mar}
\begin{tabular}{lcccc@{\hspace{1.4em}}lccccc}
\toprule
text8, bpc & 8 & 16 & 32 & 64 & MAR-B, FID & 8 & 16 & 32 & 64 & 256 \\
\midrule
contiguous & 3.14 & 2.44 & 1.77 & 1.32 & raster & 139.80 & 46.97 & 13.98 & 7.49 & 3.93 \\
confidence & 2.95 & 2.18 & 1.76 & 1.50 & random, the model's own & 13.02 & 4.10 & 2.48 & 2.32 & 2.33 \\
random & 1.80 & 1.67 & 1.62 & 1.58 & midpoint & 10.63 & 3.63 & 2.48 & 2.36 & 2.31 \\
dilated & 1.66 & 1.60 & 1.59 & 1.59 & spread & 9.44 & 3.27 & 2.40 & 2.34 & 2.28 \\
\cmidrule(r){1-5}
spread & 1.65 & 1.56 & 1.49 & 1.32 & low-discrepancy & 8.32 & 3.13 & 2.36 & 2.27 & 2.29 \\
separator & 1.63 & 1.52 & 1.38 & 1.47 & nested & 13.73 & 4.26 & 2.37 & --- & --- \\
 & & & & & spread, then raster & 17.44 & 5.80 & --- & --- & --- \\
\bottomrule
\end{tabular}
\end{table}

We rank the selection rules by their pairwise estimates, computed from
the model's kernel using Lemma~\ref{prop:pairs}. On the sequence,
contiguous selection groups the pairs with the largest kernel values.
Confidence selection favours coordinates already determined by local
context, which usually lie close together, and its estimate is large as well. Random selection mixes near and distant pairs and has an intermediate estimate. The spread rule and the separator rule keep selected pairs farther apart and have the smallest estimates. On the grid, raster order has the largest estimate, then the random and midpoint rules, and the spread, nested and low-discrepancy rules the smallest.

\begin{figure}[tb]
\centering
\includegraphics[width=0.94\textwidth]{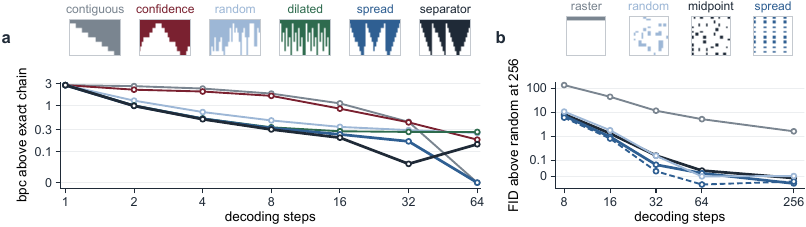}
\caption{Schedule quality versus decoding steps. (a) Excess bpc over the exact left-to-right chain on 64 held-out text8 characters after a revealed 192-character prefix, teacher forced on one weight set. Insets show the schedules. (b) Excess FID-50K over MAR-B's 256-step random order; patches show the first 32 tokens of each order and the dashed blue curve is low-discrepancy. Both vertical axes are linear below 0.1 and logarithmic above.}
\label{fig:fronts}
\end{figure}

On text8, the rules of Table~\ref{tab:predicted} rank as their pairwise estimates predict at 8, 16 and 32 steps
(Figure~\ref{fig:fronts}a, Table~\ref{tab:methods}). At 64 steps, each rule selects one
coordinate per step, so dependence cost vanishes and bpc reflects order-dependent errors
in the trained conditionals (Section~\ref{sec:magnitude}).

On MAR, FID follows the predicted ordering of the raster, random, midpoint and spread rules at 8
and 16 steps, random and midpoint tie at 32, and the low-discrepancy rule, which the estimates
place above spread, reads lowest
(Figure~\ref{fig:fronts}b, Table~\ref{tab:mar}). The spread rule's gain is largest at 8 steps, where it lowers FID by more than 25\% from the
model's own random order. Each
step then advances 32 of the 256 tokens on average, and up to 49. Randomly selecting that
many tokens places many pairs at distances where the kernel is still large, and smaller selections place fewer. At 64 steps, every
order except raster lies within 0.1 FID of the others, as their estimates
predict.

We use LLaDA to test rankings under generative perplexity and task accuracy. Confidence selection chooses tokens by the model's output, which a position-only rule ignores. The kernel predicts a gain from keeping the model's confidence ranking while excluding strongly dependent pairs, which the min-distance rule and the separator constraint do. Relative to plain confidence, the min-distance rule lowers perplexity at every measured budget, also once repetition is controlled (Appendix~\ref{app:lladappl}), and both constraints improve GSM8K accuracy from 8 to 32 steps (Table~\ref{tab:llada}). The GSM8K gain is largest at 8 and 16 steps, where steps are widest, and disappears from 64 steps on.

The cost also predicts a trend, which we verify with the above experiments: the fewer the steps, the larger the gain over the released models' own rules, in FID and GSM8K accuracy (Tables~\ref{tab:mar} and~\ref{tab:llada}).

On a Diffusion Forcing video model \citep{chen2024diffusionforcing, skyreels2025}, the schedule controls memory depth: how many earlier frames each segment conditions on. Greater depth costs more steps, and the slowly decaying inter-frame kernel predicts the same quality, as subject consistency shows (Appendix~\ref{sec:released}). Appendices~\ref{app:freegen}, \ref{app:lladappl}
and~\ref{app:humaneval} test free generation, position-only LLaDA rules and HumanEval.

\begin{table}[tb]
\centering\footnotesize\setlength{\tabcolsep}{1.4pt}
\caption{Changes relative to plain confidence on LLaDA-8B-Base. GSM8K: strict-match percentage points, 5-shot greedy decoding on 1,319 problems, paired standard errors. Perplexity: change in log generative perplexity (nats), bootstrap intervals (Table~\ref{tab:lladappl}). One 256-position block; minimum distance 6. When too few positions satisfy a constraint, confidence fills the step: 1--21\% of GSM8K steps for min-distance and 8--55\% for the separator constraint.}
\label{tab:llada}
\begin{tabular}{lccccc}
\toprule
Steps & 8 & 16 & 32 & 64 & 128 \\
\midrule
accuracy, min-distance & $+9.7 \pm 0.8$ & $+7.6 \pm 0.9$ & $+5.3 \pm 1.5$ & $-0.9 \pm 1.6$ & $-0.2 \pm 1.4$ \\
accuracy, separator constraint & $+9.7 \pm 0.8$ & $+8.2 \pm 0.9$ & $+6.5 \pm 1.5$ & $-2.0 \pm 1.6$ & $-1.4 \pm 1.3$ \\
perplexity, min-distance & $-1.04$ & $-0.90$ & $-1.49$ & $-1.56$ & --- \\
\quad 95\% interval & {\footnotesize[$-1.25$, $-0.83$]} & {\footnotesize[$-1.16$, $-0.65$]} & {\footnotesize[$-1.82$, $-1.15$]} & {\footnotesize[$-1.88$, $-1.25$]} & \\
\bottomrule
\end{tabular}
\end{table}

\subsection{The magnitude of the cost}
\label{sec:magnitude}

To see whether the estimate also gives the size of the cost, and how the cost compares with errors in the trained conditionals, we measure it on text8 as the excess bpc of a decode over the schedule's own sequential chain (Table~\ref{tab:predicted}). The estimate is at or below this excess at every budget and ranks the rules in the same order, including the spread--separator tie at 2 steps. At 2 steps it is about 30\% to 50\% of the excess; the absolute gap narrows with more steps.

\begin{table}[tb]
\centering
\begin{minipage}[c]{0.58\textwidth}
\centering\footnotesize
\renewcommand{\arraystretch}{1.2}
\setlength{\tabcolsep}{3.5pt}
\begin{tabular}{lccc}
\toprule
estimate / measured & 2 steps & 8 steps & 32 steps \\
\midrule
contiguous & 0.83 / 2.64 & 0.62 / 1.82 & 0.33 / 0.44 \\
random & 0.54 / 1.09 & 0.15 / 0.23 & 0.02 / 0.03 \\
spread & 0.32 / 0.78 & 0.06 / 0.09 & 0.00 / 0.01 \\
separator & 0.32 / 0.78 & 0.01 / 0.06 & 0.00 / 0.00 \\
\bottomrule
\end{tabular}
\end{minipage}\hfill
\begin{minipage}[c]{0.40\textwidth}
\centering
\includegraphics[width=\linewidth]{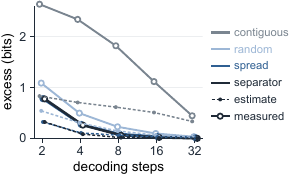}
\end{minipage}
\caption{Pairwise estimate and measured excess bpc over each schedule's own sequential chain (protocol of Figure~\ref{fig:fronts}). The plot adds 4 and 16 steps; estimates filled, measurements open.}
\label{tab:predicted}
\end{table}

A schedule's total bpc (Table~\ref{tab:methods}) minus its excess is the bpc of its sequential chain, and differences between chains are conditional errors, since text entropy is independent of order. The left-to-right chain has the lowest bpc, but contiguous selection pays the largest dependence cost, and spread a lower one with higher conditional error. When steps are wide, their difference in dependence cost is several times their difference in conditional error. When steps narrow, the random, spread and separator rules have almost the same dependence cost, but conditional errors that differ by tenths of a bit. Conditional errors explain the separator rule's rise from 32 to 64 steps and the tie of contiguous and spread at 64.

\subsection{Separation}
\label{sec:mar}

Section~\ref{sec:optimal} keeps a step cheap by separation, which gives zero cost on a Markov graph, or by distance, which only lowers it, and we test which of the two the rankings follow: on text8 one revealed character splits a run, and on MAR's token grid a split takes a whole row. On text8, a revealed character leaves most of the dependence between its neighbours within a word, and among the rules that take one coordinate per run, the separator rule, which also keeps its coordinates apart, has at most 25\% of the midpoint rule's excess (Appendix~\ref{app:separation}). On the grid, any advantage of midpoint selection comes from distance: it lowers FID below random by about 20\% at 8 steps and 10\% at 16, as the kernel's decay predicts. Spread and low-discrepancy selection lower it further, while spread-then-raster reads above random, in the order of its floors (Tables~\ref{tab:marinfo} and~\ref{tab:mar}). Below the grid's treedepth, distance sets the ranking.

Theorem~\ref{thm:separator} also sets the zero-cost budget. A chain of 64 masked text8 characters has treedepth 7, so the 8-step budget suffices on the first-order Markov reference (Corollary~\ref{cor:island}). Real text retains dependence across revealed characters (Appendix~\ref{app:separation}): at 8 steps the separator rule's measured excess is several times its estimate, and both are almost zero only at 32 (Table~\ref{tab:predicted}). MAR's $16\times16$ token grid has treedepth between 17 and 46, and under the assumptions of Theorem~\ref{thm:separator} every order has a positive cost in 8 and 16 steps, and every rule with a 64-step reading has a higher FID at both budgets. The nested rule follows the grid's separators and its 8-step floors lie below random's (Table~\ref{tab:marinfo}), yet it reads above random at both budgets (Table~\ref{tab:mar}).

\subsection{Interior levels}
\label{sec:partial}

A common observation in diffusion language models is that the absorbing channel works best \citep{austin2021d3pm, lou2024sedd, sahoo2024mdlm}, and the cost defined on the lattice gives one explanation consistent with it. An interior level, such as a token under uniform noise or a revealed character class, lowers the dependence cost of each set, since the other coordinates' interior values carry part of the dependence, but adds an advance to every coordinate, which widens every step at a fixed budget (Section~\ref{sec:conserved}); it helps when the saving exceeds the cost of the wider steps. We measure this trade-off with a graded text8 model whose interior level reveals whether a character is a space, a vowel or a consonant. Before decoding, the pairwise estimates of its two stages sum to 2 times the absorbing model's at 8 steps, where the wider steps win, and to within a few thousandths of it from 16, where the two about cancel. The absorbing model's bpc is lower at every budget (Table~\ref{tab:partial}), at 8 steps mostly through the larger dependence cost and from 16 steps mostly through the graded model's own sequential chain (Appendix~\ref{app:graded}).

\begin{table}[tb]
\centering\footnotesize
\caption{Bpc of a complete decode, all stages scored, at equal steps on held-out text8; one 85M backbone per channel.}
\label{tab:partial}
\begin{tabular}{lcccc}
\toprule
Steps & 8 & 16 & 32 & 64 \\
\midrule
absorbing, spread & 1.691 & 1.600 & 1.584 & 1.555 \\
graded, class then character & 1.790 & 1.736 & 1.716 & 1.703 \\
graded, character at once & 1.886 & 1.793 & 1.749 & 1.717 \\
\bottomrule
\end{tabular}
\end{table}

\section{Conclusion}

In this paper we ask whether the performance of the decoding schedules of one model can be predicted before decoding, and build a space and a cost that a design principle for choosing a schedule needs. The corruption lattice is a space that contains diffusion, autoregression and every schedule between them as paths through one family of transition laws. The dependence cost ranks these paths: it is the total correlation that a path's parallel steps discard, and it equals the path's divergence from exact sampling. When the data are Markov on a graph and dependent along its paths, the fewest steps of a zero-cost schedule equal the graph's treedepth for tokens and continuous values alike, which is logarithmic in the length of a sequence but linear in the side length of a grid. With even fewer steps than the treedepth, a kernel of pairwise dependence estimated from pretrained weights predicts how schedules rank before decoding. Where schedules differ in dependence cost, most measured rankings follow the predicted ones across modalities and benchmarks.

The exceptions mark what the framework leaves out. Errors in the trained conditionals lie outside the cost and can determine the ranking when dependence costs are close and reverse a prediction at one training budget that holds at a larger one (Appendices~\ref{app:frontier} and~\ref{app:cornergaps}). The pairwise bound is tight only on a first-order Markov reference. On MAR, the nested rule reads above random and the low-discrepancy rule below spread, against the order of their floors (Sections~\ref{sec:text8} and~\ref{sec:mar}, Appendix~\ref{app:marinfo}). Nonetheless, the lattice and its cost provide a quantified design principle with which future diffusion, autoregressive and intermediate decoders can be designed.

\setlength{\textfloatsep}{8pt plus 2pt minus 3pt}
\setlength{\floatsep}{6pt plus 2pt minus 2pt}
\setlength{\abovecaptionskip}{4pt}
\section*{Ethics Statement}
We use public datasets and released model weights, and train our own models on text8, CIFAR-10 and ImageNet-256. On released models we change only how they decode, and what they can produce stays as released.

\section*{Reproducibility Statement}
Appendices~\ref{app:lattice} and~\ref{app:design} contain the formal statements and proofs. Table~\ref{tab:rules}, Section~\ref{sec:optimal} and Appendix~\ref{app:marinfo} define the decoding rules, and the table and figure captions state the evaluation protocols. We evaluate released models with their released weights. The supplementary code trains our models, estimates the kernels and decodes under every schedule. Each table and figure that reports a measurement has one script, named by its number in the paper, that writes its data.

\section*{AI Use Statement}
We used generative AI tools to polish the manuscript and check novelty against existing publications. We also used them for coding and for checking derivations. We have reviewed all AI-assisted work and take responsibility for the content of the paper.

\bibliography{references}
\bibliographystyle{preprint}

\appendix

\section{Related Work}
\label{sec:related}

\subsection{Hybrids of autoregression and diffusion}
Order-agnostic autoregression and absorbing diffusion are special cases of autoregressive diffusion models, which support parallel generation adapted to a given budget of steps \citep{hoogeboom2022ardm}, and the training objective of absorbing diffusion is that of an any-order autoregressive model \citep{ou2025absorbing}. Block diffusion generates autoregressively across blocks and uses
diffusion within each block \citep{arriola2025bd3lm}. Draft-verify
systems propose tokens with a diffusion head and accept them under an
autoregressive model \citep{tidar2025}. Their deployed distribution
remains autoregressive, with diffusion accelerating the proposals.
Esoteric language models interpolate between the two objectives using
a masking-rate hyperparameter \citep{sahoo2025esolm}.

Hyperschedules assign each token position a monotone noise schedule,
with autoregression and diffusion as members of the resulting family
\citep{fathi2025hyperschedules}. Each schedule is a path on our lattice.
The authors also study block and sliding-window schedules that advance
from left to right on discrete text, obtaining each by fine-tuning a
model trained with the diffusion schedule. AR-Diffusion assigns fewer
denoising steps to tokens on the left, which therefore resolve first
\citep{wu2023ardiffusion}. Both works train under the likelihood bound
of the chosen schedule. Both show that autoregression can be written as a diffusion schedule. Neither quantifies the
dependence discarded when selected coordinates are sampled independently, so each schedule
must be trained before it can be compared. Proposition~\ref{thm:pathcost}
gives this quantity for every path. With a kernel from an existing
model, we can compare schedules before decoding or before training the
target model. The chosen schedule runs on the pretrained weights;
every schedule we compare on released models in
Section~\ref{sec:frontier} uses those weights as released.

Video models also use intermediate corruption levels. Diffusion
Forcing trains with an independent level per coordinate on a causal
architecture, motivated as uncertainty-aware next-token prediction
\citep{chen2024diffusionforcing}. We represent these corruption
patterns on a lattice and derive the dependence cost of each step,
including steps over spatial coordinates. Rolling schedules and
diagonal denoising fix the slope in advance
\citep{ruhe2024rolling, kim2024fifo}; streaming windows carry
monotonically increasing levels \citep{liu2025rollingforcing}. These
arrangements are paths on the same lattice (Figure~\ref{fig:hero}),
whose per-step dependence cost Section~\ref{sec:design} defines. For training across paths, we use two-stream
attention from the order-modelling literature \citep{yang2019xlnet}
to recover the supervision density of autoregression.

\subsection{The dependence cost of a decoding step}
\label{sec:rw:price}
On the absorbing lattice with every advanced set fixed, Proposition~\ref{thm:pathcost} is due to \citet{chen2025schedules}. They sharpen the parallel-decoding bound of \citet{li2025convergence} to an identity: the divergence between the data and the finished sample is the sum over steps of the total correlation of the advanced coordinates conditioned on the revealed ones. Drawing the positions of each step uniformly among those still masked, they reduce the expected divergence to the distance between a curve of mean mutual information and its step approximation, choose how many coordinates each step advances, and derive $O(\log d)$ schedules for $d$ tokens in terms of the total and dual total correlation of the data. The averaging over positions removes the geometry that the kernel of Section~\ref{sec:choose} keeps: their $O(\log d)$ steps bound the error for data of small total or dual total correlation, and the treedepth of Theorem~\ref{thm:separator} counts the steps of exact sampling for data Markov on a graph and dependent along its paths, where the positions each step updates decide the cost. Proposition~\ref{thm:pathcost} extends the identity to every monotone path on the lattice, with coordinates at different levels, and to continuous channels. It gives the dependence cost of each schedule, allowing us to compare autoregressive and diffusion paths with block, sloped and confidence schedules on both tokens and continuous fields.

\citet{luxembourg2025dus} study coordinate selection with a fixed number of steps per
block. They unmask dilated groups of non-adjacent positions in
parallel, minimise an upper bound on joint entropy gain, and outperform
confidence planners on reasoning and code benchmarks at deployment
scale. The selection pattern is fixed before decoding. Independence
comes from distance: on a fast-mixing Markov chain, well-separated
positions have a joint entropy within $\varepsilon$ of the sum of
their marginal entropies. Section~\ref{sec:optimal} obtains exact independence from separation: when the data are Markov on a graph, revealed coordinates divide it into components, and selecting one coordinate per component gives zero total correlation at any state. The two mechanisms make different predictions on a grid, which Section~\ref{sec:mar} tests. Parallel Gibbs sampling updates the variables of one colour class of a Markov random field at once, since they are independent given all their neighbours \citep{gonzalez2011parallel}. A decoding step conditions only on the revealed coordinates, and its zero-cost sets are those separated by revealed coordinates. The fewest such steps are the treedepth of Theorem~\ref{thm:separator}, while a chain or a grid needs only 2 colours.

Confidence-based parallel decoding dates to MaskGIT
\citep{chang2022maskgit}. Subsequent work replaces confidence selection
with supervised planners \citep{asano2026whereto}, learned policies
\citep{jazbec2025unmasking}, or lookahead over unmasking paths
\citep{lee2025lookum}. These methods train a planner or require extra
forward passes per step. The rules we compare require a sort or a
position computation per step. Dilated selection uses a pattern fixed before decoding;
Appendix~\ref{app:overhead} times the other rules. \citet{cai2026confidence} prove
acceleration for an entropy-capped rule that follows a uniformly
random scan, unmasking until the sum of the batch's conditional
entropies crosses a threshold. Their theorem derives and controls
\eqref{eq:tc} summed along the schedule as the exact cost of parallel
unmasking. The averaging argument requires a random scan and does not
cover the sorted rule used by deployed decoders. Moreover, the entropy
cap uses the marginal entropies in \eqref{eq:tc}, while the cost also
depends on the joint entropy of the selected coordinates. We measure the dependence cost of selecting coordinates by confidence in Appendix~\ref{app:design} (Tables~\ref{tab:secondorder} and~\ref{tab:conf}).

\section{The Lattice and Its Corners}
\label{app:lattice}

\noindent\textbf{Lemma~\ref{prop:paths} (Path representation).}
\emph{Let $F$ be the family of transition laws $p_{\ell\to\ell'}$ between lattice points $\ell' \preceq \ell$. Every monotone path composes elements of $F$ into a law whose marginal on the clean coordinates is the data distribution, and two paths differ only in which elements of $F$ they evaluate.}
\begin{proof}
A monotone path visits states
$\ell^{(0)} = \top \succ \dots \succ \ell^{(T)} = \bot$. The joint law of the states along this path factorises into
transition laws. The factor at step $t$ is by definition
the transition law $p_{\ell^{(t)}\to\ell^{(t+1)}}$, an element of $F$
indexed by the source and destination lattice points. Composed from $\top$, the factors run the corruption
process in reverse. The marginal on the clean coordinates is the
data law. When the corrupted state is a function of the data and the
levels, the same composition is a factorisation of the data
likelihood itself. Each factor depends on its two endpoint states.
Two paths through the same lattice draw their factors from the
same family.
\end{proof}

The synchronised path gives a likelihood bound, and the sequential path the exact likelihood, which the AR-chain rows of Figure~\ref{fig:corners} evaluate and which continuous-time diffusion computes through its probability-flow ODE \citep{song2021sde, chen2018neural}.

We define the corner quality gap at a corner as the difference in quality between a model trained on the whole lattice and one trained only for that corner.

\begin{corollary}[No representational corner quality gap]
\label{prop:notax}
A model that realises $F$ realises the conditionals of every path.
Any measured corner quality gap is attributable to finite capacity or optimisation, to the allocation of the training distribution over paths, or to
which conditionals the architecture can supervise at a lattice point, and not to the
coexistence of the two corners in one weight set.
\end{corollary}
\begin{proof}
By Lemma~\ref{prop:paths} the conditionals of every path are
elements of one family $F$. A model that realises $F$ evaluates the
conditionals of every path by restriction. A trained model departs from $F$ only through a model class that does not contain $F$ or training that does not reach it, through the probability its training distribution assigns to each path, or through an architecture that cannot supervise the conditionals at a lattice point. When these restrictions are absent, the realised conditionals coincide with $F$ at both corners.
\end{proof}

\subsection{The interior states}
\label{app:gradedstates}

Figure~\ref{fig:gradedstates} shows the pixel model's intermediate
states under 5 schedules. The synchronised path denoises the whole
field at a common level, whereas a sloped path produces a visible
noise gradient across the image at a given step. Applying the sloped
path in spread order denoises separated patches first, leaving noisy
regions between the clean ones. The sequential path denoises one patch
at a time.

\begin{figure}[htbp]
\centering
\includegraphics[width=\textwidth]{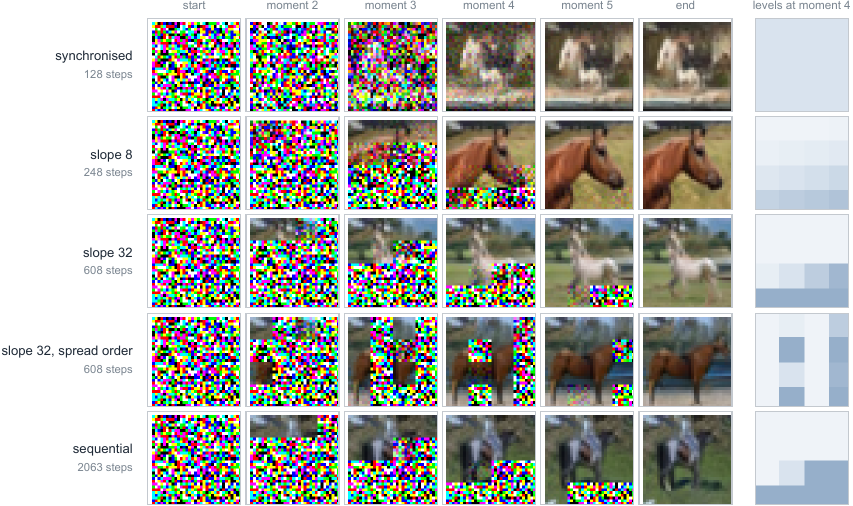}
\caption{Intermediate states of the CIFAR-10 lattice model with
16 patches of 8 by 8 and 128 transitions. Each row shows one of five
schedules at 6 points along its path; all rows use the same seed.
The last column gives the level map at the fourth point, from white
(clean) to blue (corrupted). On the grid of 128 transitions, the mean
noise-level index in each patch column decreases from left to right
by 19 at slope 8, by 24 at slope 32 and by 32 along the sequential
path. It is constant along the synchronised path. In spread order,
the 4 column means are 128, 47, 128 and 15.}
\label{fig:gradedstates}
\end{figure}

\subsection{The corner quality gaps}
\label{app:cornergaps}

Figure~\ref{fig:corners} shows the first 2 pixel samples at a fixed seed and reads CIFAR-10 FID-50K with one evaluator. Its text8 models at 85M share data, budget and seed, and are teacher forced after 192 revealed characters. The lattice model's training distribution splits every batch equally among prefix states, random-subset masking and intermediate widths, the last with a co-trained head that predicts the advanced coordinates jointly. Each model is read at its validation minimum: 25K training steps for the causal model, 165K for the dedicated masked model, 185K for the two-stream model and 510K for the lattice model. The dedicated masked model is clearly better at its minimum than at its last checkpoint, while the lattice and two-stream models read the same at both. The masked-ELBO rows average the chain over random orders, which the causal backbone cannot evaluate.

At the synchronised corner, the dedicated model's FID is about 2 points lower; the lattice model assigns 25\% of its training distribution there (Table~\ref{tab:pixelcorners}). Corollary~\ref{prop:notax} separates what such a gap can come from.

The text8 gaps depend on which conditionals a backbone supervises in one forward pass (Appendix~\ref{app:readable}). At the sequential corner, the two-stream model's bpc lies 0.03 bits above the dedicated causal model's. Under the
masked-diffusion bound, the lattice model's gap to the dedicated masked model is 0.13 bits (Figure~\ref{fig:corners}).

For the mixed model, replacing raster order with spread order lowers
FID at slope 4 but raises it at slope 16
(Table~\ref{tab:pixelcorners}). The dependence cost favours spread order at both slopes, since it keeps
apart the patches that advance together. The held-out objective has been flat
since 60K training steps, and continuing the model to 120K steps on fresh data changes it by less than 0.1\%. With this doubled budget, the spread rule has lower FID than raster
order at slope 16, over 10K samples and over 50K (Table~\ref{tab:pixelcorners}).

On ImageNet-256 we repeat the comparison with 2 SiT-B/2 models
(Table~\ref{tab:imagenet}). At the synchronised corner the lattice
model's FID is within 4 of the dedicated model's. At the sequential
corner the synchronised-only model collapses, and the lattice model's
FID is more than 2 times its own synchronised value and above that of its
synchronised path at 16 steps, which updates each coordinate as often.
One weight set reaches both corners at unequal quality.

\begin{table}[htbp]
\centering\small
\caption{CIFAR-10 FID over 50K samples for the models in
Figure~\ref{fig:corners}. Rows specify the training distribution and
columns the decoding schedule; the two lattice distributions appear
above the two single-schedule controls. On a sloped path each patch starts its 128 transitions a fixed number of steps after the patch
before it, the slope $s$, and the path takes $128 + 15s$ steps. Each column labelled spread uses
the sloped path at the indicated slope, with patches selected by the spread rule in place of
raster order. The first lattice distribution draws each
coordinate's level independently, reaching a corner with vanishing
probability. The second assigns 25\% of its mass to each corner; its row at 120K steps continues
the same model on fresh data.}
\label{tab:pixelcorners}
\footnotesize
\setlength{\tabcolsep}{2.6pt}
\begin{tabular}{lccccccc}
\toprule
& & & \multicolumn{2}{c}{slope 4} & \multicolumn{2}{c}{slope 16} & \\
\cmidrule(lr){4-5}\cmidrule(lr){6-7}
& sync. & slope 1 & raster & spread & raster & spread & seq. \\
steps & 128 & 143 & 188 & 188 & 368 & 368 & 2063 \\
\midrule
iid levels & 15.8 & 15.5 & 15.2 & 14.0 & 13.1 & 12.0 & 13.0 \\
mixed levels & 9.8 & 10.1 & 12.2 & 11.0 & 11.6 & 12.1 & 11.9 \\
\quad at 120K steps & --- & --- & --- & --- & 11.1 & 10.5 & --- \\
synchronised only & 7.6 & 20.7 & 187.2 & 171.2 & 251.6 & 241.0 & 273.4 \\
sequential only & 299.2 & 278.6 & 242.3 & 272.4 & 64.8 & 189.9 & 40.9 \\
\bottomrule
\end{tabular}
\end{table}

\begin{table}[htbp]
\centering\small
\caption{FID-50K on ImageNet-256 against the ADM reference for 2 SiT-B/2 models trained for 400K steps at batch size 256. The models
differ only in their training distribution over levels and are sampled at guidance scale 1, using 1 seed. At the sequential corner, both
follow the same path through 16 groups of 16 patches. The final row
uses the synchronised path, updating each coordinate as often as at
the sequential corner.}
\label{tab:imagenet}
\begin{tabular}{lcc}
\toprule
& synchronised & sequential \\
steps & 249 & 256 \\
\midrule
lattice, mixed levels & 43.3 & 98.7 \\
synchronised only & 39.9 & 205.5 \\
\midrule
lattice, synchronised at 16 steps & 54.5 & --- \\
\bottomrule
\end{tabular}
\end{table}

\FloatBarrier
\subsection{Lattice points an architecture can supervise}
\label{app:readable}

A causal model computes every conditional along its path in one
forward pass because each position depends only on its prefix. A backbone that takes mask tokens as input computes coordinate
conditionals at a single lattice state per pass. Consequently, at the
same wall-clock time, the lattice model trains on orders of magnitude
fewer autoregressive targets than the dedicated causal model.

We increase the supervision density with two-stream attention. We sample an order
$\pi$ and give the backbone two streams, one carrying the true tokens
and one carrying positions, with attention masked so that a query at
rank $j$ attends to content at ranks below $j$ \citep{yang2019xlnet}. One
forward pass then supervises every conditional along the sampled path.
Under the identity order the construction reduces to autoregression
exactly.

Of the models in Figure~\ref{fig:corners}, the lattice and dedicated masked models supervise coordinate conditionals at one lattice state per pass, and at the sequential corner the lattice model's bpc is 0.06 bits lower. The causal and two-stream backbones supervise every conditional along a path and lower the bpc by about 0.1 bits more.
The two-stream backbone also evaluates the masked-diffusion bound by
using its query stream with every revealed coordinate placed before
every masked coordinate. Its masked ELBO is 0.2 bits above
the lattice model's. Dense supervision improves the sequential corner but not the masked-diffusion bound.

We compare the advances available from a lattice state under each
architecture (Table~\ref{tab:readable}). An architecture expresses an
advance when one forward pass provides the law of every advanced
coordinate. A causal mask allows width-1 advances from prefix
states. A bidirectional backbone computes coordinate conditionals at
any one lattice state per pass. Two-stream attention also accepts any
lattice state and supervises every conditional along the sampled path
in one pass. For a set of several coordinates, a product law and a
sequential chain evaluate different distributions, so we score each
architecture under the law it expresses for that advance. To compare
supervision density, we use width-1 advances, where these laws
coincide. This tests the supervision term in
Corollary~\ref{prop:notax}.

\begin{table}[tb]
\centering\small
\caption{Advances from lattice states evaluated under each architecture at matched data,
budget and size, 16 seeds with one shared conditioning pattern
per seed. A cell reports the loss under the law expressed by
its architecture: a sequential chain for the causal and two-stream rows
and a product law for the bidirectional row. We report the mean and
standard deviation over the 16 draws. An empty cell is an advance the architecture
cannot express. Variation in the interior width-1 column reflects
position-to-position differences in data entropy. The paired
excess of the two-stream value over the bidirectional value in the
interior width-1 column is 0.27$\pm$0.08 bits, positive in every draw.}
\label{tab:readable}
\begin{tabular}{lccc}
\toprule
& prefix, width 1 & interior, width 1 & interior, width 8 \\
\midrule
causal & 1.22$\pm$0.11 & --- & --- \\
bidirectional & 1.41$\pm$0.10 & 2.00$\pm$1.10 & 1.01$\pm$0.31 \\
two-stream & 1.26$\pm$0.10 & 2.26$\pm$1.13 & 1.31$\pm$0.37 \\
\bottomrule
\end{tabular}
\end{table}

\section{The Schedule on the Lattice}
\label{app:design}

\noindent\textbf{Proposition~\ref{thm:pathcost} (Exact path cost of parallel advance).}
\emph{Let $\pi = (\ell^{(0)}, \dots, \ell^{(T)})$ be a monotone schedule on
the corruption lattice with advanced sets $S_t$. Write $x_{S_t}$ for
the updated values at the destination levels $\ell^{(t+1)}_{S_t}$,
with the transition index fixed by $\pi$ and suppressed. Let $P^\pi$ be the law of the path that draws
each step from the exact joint law $p(x_{S_t} \mid \ell^{(t)})$, and
$\widetilde P^\pi$ the law of the path that draws each step from the
product $\prod_{i \in S_t} p(x_i \mid \ell^{(t)})$ of the same
coordinate conditionals. Then
$D_{\mathrm{KL}}(P^\pi \,\|\, \widetilde P^\pi) =
\sum_{t=0}^{T-1} \mathbb{E}_{P^\pi}[\mathrm{TC}(S_t \mid
\ell^{(t)})]$ with $\mathrm{TC}(S \mid \ell) = \sum_{i \in S} H(x_i \mid \ell)
- H(x_S \mid \ell)$,
and the divergence between the two laws of the finished sample is at
most the same sum.}
\begin{proof}
Both $P^\pi$ and $\widetilde P^\pi$ are Markov chains with the same
initial state $\top$ and the same sequence of level assignments $\pi$.
They differ only in the step kernel. By the chain rule for relative
entropy over a Markov chain,
\[
D_{\mathrm{KL}}\big(P^\pi \,\|\, \widetilde P^\pi\big)
= \sum_{t=0}^{T-1} \mathbb{E}_{P^\pi}\Big[
D_{\mathrm{KL}}\Big(p(x_{S_t} \mid \ell^{(t)}) \,\Big\|\,
\prod_{i \in S_t} p(x_i \mid \ell^{(t)})\Big)\Big],
\]
the inner divergence taken at the state the exact chain reaches at
step $t$. The divergence of a joint law from the product of its own
marginals is the total correlation, which gives the identity. The
finished sample is a function of the path, and relative entropy does
not increase under a function, which gives the bound. The argument applies to both channels and all levels of the coordinates in $S_t$:
the step kernel at a graded lattice point is the joint law of the
advanced values given the state, and its product form is the product
of the same conditionals. For a singleton $S_t$, the joint law equals its sole marginal, so
the total correlation is zero. Two schedules over the same $F$ enter the sum only through their
advanced sets $S_t$, the states from which they advance and their
destination levels.
\end{proof}

\noindent\textbf{Lemma~\ref{prop:pairs} (The pairwise lower bound).}
\emph{For an advanced set $S = \{i_1 < \dots < i_k\}$, at any state
and under any data distribution,
$\mathrm{TC}(S \mid \ell) \ge \sum_{j=2}^{k}
I(x_{i_j};\, x_{i_{j-1}} \mid \ell)$,
the kernel summed over consecutive members of $S$. Equality holds
exactly when, given the state, each selected coordinate is
independent of the earlier selected ones given its predecessor: a
first-order Markov reference.}
\begin{proof}
By the chain rule for mutual information,
\[
\mathrm{TC}(S \mid \ell)
= \sum_{i \in S} H(x_i \mid \ell) - H(x_S \mid \ell)
= \sum_{j=2}^{k} I\big(x_{i_j};\, x_{i_1}, \dots, x_{i_{j-1}} \,\big|\, \ell\big).
\]
Mutual information is monotone in its second argument, so the $j$th
term is at least $I(x_{i_j};\, x_{i_{j-1}} \mid \ell)$, which is the
bound. The difference in the $j$th term is
$I(x_{i_j};\, x_{i_1}, \dots, x_{i_{j-2}} \mid x_{i_{j-1}}, \ell)$.
It is nonnegative and equals zero exactly when $x_{i_j}$ is independent
of the earlier selected coordinates given its predecessor and the
state, as in the stated reference model. A first-order Markov chain satisfies this conditional-independence condition at every state. Conditioning a Markov chain on the values at a fixed
set of sites preserves the Markov property over the remaining sites,
and every selected position other than $i_{j-1}$ in the $j$th term
lies before it.
\end{proof}

\noindent\textbf{Theorem~\ref{thm:separator} (Separator schedules).}
\emph{Fix a lattice state and suppose that, given its revealed coordinates, the unresolved coordinates are Markov with respect to a graph $G$, and that at this and every later state any two unresolved coordinates joined by a path of unresolved coordinates are dependent. Then the fewest steps of a schedule with zero dependence cost equal the treedepth $\mathrm{td}(G)$ of $G$.}
\begin{proof}
Condition on the revealed coordinates. The unresolved coordinates are
Markov on $G$, so coordinates in distinct connected components are
independent given the revealed ones. A step that advances at most one
coordinate per component advances a set of mutually independent
coordinates, whose joint law is the product of its marginals, and its
total correlation is zero. On a chain the unresolved coordinates
form runs whose boundaries are revealed, and each run is one component.
Resolving the midpoint of every run advances one coordinate per
component, so the step costs zero, and it splits a run of length $g$
into two runs of length at most $\lfloor g/2 \rfloor$. A run of length
$d$ is resolved after $\lceil \log_2 (d+1) \rceil$ such steps. 

For the number of steps, let $f(G)$ be the least number of zero-cost steps that resolve every coordinate of $G$. A coordinate leaves the graph only when it is resolved: revealing it separates its neighbours, while a coordinate at an interior level stays unresolved and leaves its component connected. A single coordinate needs one step. On a disconnected graph the components are independent given the revealed coordinates and are resolved side by side, so $f(G) = \max_C f(C)$ over the components $C$. On a connected graph a step that advances only one coordinate $v$, to the clean level, costs zero and leaves $G - v$, so $f(G) \le 1 + \min_v f(G - v)$. These recursions define the treedepth \citep{nesetril2006treedepth}, and by induction on the number of coordinates $f(G) \le \mathrm{td}(G)$, which needs only the Markov property. Under the dependence assumption, a step that resolves two coordinates of one component has a joint law that is not the product of its marginals and a positive total correlation. A step that moves coordinates only to interior levels leaves all of them unresolved and, by the same assumption, leaves every component connected. A zero-cost step on a connected graph therefore resolves at most one coordinate, $f(G) \ge 1 + \min_v f(G - v)$, and $f(G) = \mathrm{td}(G)$; every schedule with fewer steps contains a step of positive cost. On a path of $d$ coordinates $\mathrm{td} = \lceil \log_2 (d+1) \rceil$ \citep{nesetril2012sparsity}, which the midpoint schedule attains as shown above. On the $n \times n$ grid the treedepth exceeds the treewidth \citep{nesetril2012sparsity}, which is $n$ \citep{bodlaender1998arboretum}. Deleting the middle line across the longer side of each remaining rectangle and recursing on the two halves, as in nested dissection \citep{george1973nested}, resolves the grid with zero cost, which gives the upper bound: 46 steps for $n = 16$, about $3n$ in general.
\end{proof}

On text, a step that advances one coordinate per run pays the conditional total correlation of Proposition~\ref{thm:pathcost} across the separators. Appendix~\ref{app:separation} measures it through the separator and midpoint rules, which advance one coordinate per run whenever the runs allow.

\noindent\textbf{Corollary~\ref{cor:island} (The reference model).}
\emph{On the reference model of Lemma~\ref{prop:pairs}, a first-order
Markov chain at every state, a step of $k$ coordinates with consecutive
gaps at least $s$, where $\kappa \le \varepsilon$ at distances $\ge s$, pays at
most $(k-1)\,\varepsilon$. A schedule whose first step advances
$S_1 = \{i_1 < \dots < i_m\}$ and whose every later step advances the
midpoint of every run of masked coordinates pays
$\mathrm{TC}(S_1 \mid \top) = \sum_{j=2}^{m} \kappa(i_{j-1}, i_j)$ and
completes within budget $T$ when every gap of $\{0\} \cup S_1 \cup \{d+1\}$ is at most $2^{\,T-1}$.}
\begin{proof}
A first-order Markov chain is Markov on the chain graph, so a step
whose coordinates are pairwise separated by revealed coordinates
advances one coordinate per component and has zero cost, since coordinates in different components are independent given the revealed ones (Section~\ref{sec:optimal}). For coordinates in one run, the bound of Lemma~\ref{prop:pairs} is an equality on the reference model, so the
cost is the sum of the $k-1$ kernel terms. When every gap is at least
$s$, each term is at most $\varepsilon$. Lemma~\ref{prop:pairs} gives the first step's cost at the empty
state, and the separation just shown gives zero cost at every later
step. Once its boundaries are revealed, a gap of size $g$ between consecutive members of
$\{0\} \cup S_1 \cup \{d+1\}$, whose two added positions stand for the ends of the chain, is
resolved by repeatedly selecting its midpoint, in
$\lceil \log_2 g \rceil$ further steps with at most one of its
coordinates selected per step.
\end{proof}

\subsection{The text8 kernel and the pairwise estimates}
\label{app:rates}

We tabulate the kernel of Figure~\ref{fig:kernel}a (Table~\ref{tab:kernel}) and measure it over the mask rates encountered during decoding (Table~\ref{tab:kernelrates}) to compute the estimates in Table~\ref{tab:predicted}.

Prediction requires a schedule and a kernel. We test whether the
kernel must come from the model that will use the schedule by
computing the estimates with the kernel of our 26M text8 model and comparing them
with the costs measured on the 85M model
(Table~\ref{tab:smallkernel}). The smaller model predicts the larger
model's ranking at every budget. Here the smaller model's kernel would have chosen the
larger model's schedule before the larger model was trained. We also evaluate each
rule's sequential chain on the 85M model (Table~\ref{tab:chains})
and add the triplet term to the pairwise estimate
(Table~\ref{tab:secondorder}).

\begin{table}[htbp]
\centering\small
\caption{The pair kernel $\kappa$: the decrease in entropy at
one masked coordinate when a second coordinate at the given distance
is revealed. Values are in bits at mask rate 0.5, mean $\pm$ standard deviation over 3 seeds with 512 pairs per cell. A cell is left blank when too few pairs are
available.}
\label{tab:kernel}
\begin{tabular}{lccccc}
\toprule
Distance & 1 & 2 & 4 & 8 & 16 \\
\midrule
inside one word & 0.996 $\pm$ 0.050 & 0.694 $\pm$ 0.114
  & 0.505 $\pm$ 0.030 & 0.187 $\pm$ 0.021 & --- \\
across a boundary & 0.368 $\pm$ 0.086 & 0.269 $\pm$ 0.077
  & 0.121 $\pm$ 0.058 & 0.023 $\pm$ 0.032 & 0.009 $\pm$ 0.005 \\
\bottomrule
\end{tabular}
\end{table}

Across the mask rates encountered during decoding, the spread rule's advanced sets have lower total correlation than the random rule's at every width, by an order of magnitude at width 8 and a factor of 3 to 5 at width 32 (Table~\ref{tab:rates}).

\begin{table}[htbp]
\centering\small
\caption{Total correlation of the advanced set at 3 mask rates,
in $10^{-3}$ bits per coordinate over 3 seeds.}
\label{tab:rates}
\begin{tabular}{llccc}
\toprule
Rule & mask rate & width 8 & width 16 & width 32 \\
\midrule
random & 0.3 & 26.4 $\pm$ 12.5 & 63.4 $\pm$ 9.8 & 119.7 $\pm$ 1.0 \\
random & 0.5 & 73.1 $\pm$ 7.7 & 144.6 $\pm$ 12.8 & 268.1 $\pm$ 18.6 \\
random & 0.7 & 102.7 $\pm$ 18.6 & 228.7 $\pm$ 9.2 & 420.6 $\pm$ 20.0 \\
\midrule
spread & 0.3 & 2.7 $\pm$ 3.2 & 6.7 $\pm$ 2.5 & 25.0 $\pm$ 3.3 \\
spread & 0.5 & 4.4 $\pm$ 1.0 & 13.1 $\pm$ 1.1 & 68.6 $\pm$ 10.6 \\
spread & 0.7 & 6.1 $\pm$ 2.3 & 27.4 $\pm$ 8.6 & 130.7 $\pm$ 2.7 \\
\bottomrule
\end{tabular}
\end{table}

\begin{table}[htbp]
\centering\small
\caption{The pair kernel in Table~\ref{tab:kernel} at 5 mask rates, in bits, averaged over 3 seeds. At each distance, we weight
pairs inside a word and pairs across a boundary by their proportions
in the corpus. To predict the costs in Table~\ref{tab:predicted},
we evaluate the kernel at each step's mask rate, interpolating linearly
between measured rates and using the nearest value outside their
range.}
\label{tab:kernelrates}
\begin{tabular}{lcccc}
\toprule
Mask rate & $d=1$ & $d=2$ & $d=4$ & $d=8$ \\
\midrule
0.1 & 0.216 & 0.115 & 0.020 & 0.003 \\
0.2 & 0.298 & 0.162 & 0.051 & 0.000 \\
0.3 & 0.425 & 0.289 & 0.093 & 0.013 \\
0.5 & 0.785 & 0.480 & 0.214 & 0.030 \\
0.7 & 0.936 & 0.646 & 0.331 & 0.118 \\
\bottomrule
\end{tabular}
\end{table}

\begin{table}[htbp]
\centering\small
\caption{Pairwise contribution estimated from the 26M model's
kernel for the schedules and mask rates in
Table~\ref{tab:predicted}, compared with the 85M model's measured
excess over each schedule's sequential chain. Values are in bits per
character.}
\label{tab:smallkernel}
\begin{tabular}{lccccc}
\toprule
predicted / measured & 2 steps & 4 steps & 8 steps & 16 steps & 32 steps \\
\midrule
contiguous & 0.92 / 2.64 & 0.80 / 2.34 & 0.71 / 1.82 & 0.59 / 1.12 & 0.39 / 0.44 \\
random & 0.58 / 1.09 & 0.33 / 0.50 & 0.16 / 0.23 & 0.07 / 0.10 & 0.02 / 0.03 \\
spread & 0.31 / 0.78 & 0.11 / 0.27 & 0.07 / 0.09 & 0.04 / 0.03 & 0.01 / 0.01 \\
separator & 0.31 / 0.78 & 0.08 / 0.26 & 0.01 / 0.06 & 0.01 / 0.02 & 0.00 / 0.00 \\
\bottomrule
\end{tabular}
\end{table}

\begin{table}[htbp]
\centering\small
\caption{Bits per character of each rule's sequential chain on
the 85M model, for the decodes in Table~\ref{tab:predicted}. We reveal
the coordinates in each selected set one at a time, following the
schedule's order. The entropy of the text does not depend on this
order; differences between rows arise from the order-dependent error
of the trained conditionals.}
\label{tab:chains}
\begin{tabular}{lccccc}
\toprule
Steps & 2 & 4 & 8 & 16 & 32 \\
\midrule
contiguous & 1.32 & 1.32 & 1.32 & 1.32 & 1.32 \\
confidence & 1.40 & 1.38 & 1.35 & 1.44 & 1.48 \\
random & 1.51 & 1.55 & 1.57 & 1.57 & 1.58 \\
spread & 1.53 & 1.57 & 1.56 & 1.52 & 1.48 \\
separator & 1.53 & 1.57 & 1.56 & 1.50 & 1.38 \\
\bottomrule
\end{tabular}
\end{table}

\begin{table}[htbp]
\centering\small
\caption{The second-order estimate and measured excess, in bits per
character under the convention of Table~\ref{tab:predicted}. We add the
pairwise estimate in that table to the triplet term
$I(x_c;\,x_a \mid x_b,\ell)$ of the proof of Lemma~\ref{prop:pairs},
estimated on consecutive triples of the same schedule with no revealed coordinate between them, each step at
its own mask rate, with 512 triples per cell over 3 seeds. For confidence selection, we compute the estimate after decoding from the sets selected during decoding. We compare it with the measured excess bpc over the sequential chain that reveals those same coordinates in the same order.}
\label{tab:secondorder}
\setlength{\tabcolsep}{5.5pt}
\begin{tabular}{lccccc}
\toprule
second-order estimate / measured & 2 steps & 4 steps & 8 steps & 16 steps & 32 steps \\
\midrule
contiguous & 1.28 / 2.64 & 1.07 / 2.34 & 0.90 / 1.82 & 0.69 / 1.12 & 0.33 / 0.44 \\
random & 0.76 / 1.09 & 0.39 / 0.50 & 0.17 / 0.23 & 0.07 / 0.10 & 0.02 / 0.03 \\
spread & 0.48 / 0.78 & 0.12 / 0.27 & 0.06 / 0.09 & 0.03 / 0.03 & 0.00 / 0.01 \\
separator & 0.48 / 0.78 & 0.10 / 0.26 & 0.02 / 0.06 & 0.01 / 0.02 & 0.00 / 0.00 \\
\midrule
confidence, measured after the decode & 1.12 / 2.13 & 0.94 / 1.96 & 0.81 / 1.60 & 0.42 / 0.74 & 0.16 / 0.28 \\
\bottomrule
\end{tabular}
\end{table}

\subsection{Separation on text}
\label{app:separation}

After its first step, the midpoint rule advances at most one coordinate per masked run, which
leaves a revealed coordinate between every pair it advances. The separator rule does the same except
when a step needs more coordinates than there are runs, which after the first step happens in the
second step at 8 steps and in every step at 16. The first-order Markov reference of Lemma~\ref{prop:pairs}
charges zero to a pair with a revealed coordinate between them and gives the two rules
estimates within 0.02 bits. On text8, the midpoint rule's measured excess is at least 4 times the separator rule's from 8 steps on (Table~\ref{tab:separation}).

The two rules place their coordinates differently. Most pairs of consecutive coordinates in a
midpoint step sit 2 or 4 positions apart, with a single revealed character between them, whereas the
separator rule places its coordinates as far apart as the runs allow. We read the kernel with the
characters between a pair controlled: exactly $s$ of them are revealed, and the rest of the window
is masked at the step's rate. Within a word, one revealed character between two masked characters
leaves more than 50\% of their dependence at every distance we read
(Table~\ref{tab:betweenkernel}). On text, a single revealed character therefore does not separate
its neighbours in the sense of Theorem~\ref{thm:separator}.

Charging every pair the kernel at its distance and its number of revealed characters between, in
place of the reference's zero, gives the midpoint rule most of its measured excess and keeps the
separator rule's estimate within 0.01 bits of its measured excess
(Table~\ref{tab:separation}). We also computed this estimate before decoding a third rule, which
advances the midpoints of the longest runs and spreads its choices among runs of equal length. The
estimate placed that rule's excess below the midpoint rule's at 8, 16 and 32 steps, as its decode
then measured.

\begin{table}[htbp]
\centering\small
\caption{Text8 kernel of the 85M model at mask rate 0.5, with $s$ of the characters between
the pair revealed and the rest of the window masked at that rate: entropy decrease in bits when the
second coordinate is revealed, mean over 3 seeds with 512 pairs per cell and seed.}
\label{tab:betweenkernel}
\begin{tabular}{lcccc}
\toprule
Distance, pair & $s=0$ & $s=1$ & $s=2$ & $s=3$ \\
\midrule
2, within a word & 0.83 & 0.51 & --- & --- \\
2, across a boundary & 0.34 & 0.17 & --- & --- \\
3, within a word & 0.81 & 0.60 & 0.36 & --- \\
3, across a boundary & 0.32 & 0.18 & 0.08 & --- \\
4, within a word & 0.69 & 0.66 & 0.43 & 0.21 \\
4, across a boundary & 0.25 & 0.16 & 0.12 & 0.04 \\
\bottomrule
\end{tabular}
\end{table}

\begin{table}[htbp]
\centering\small
\caption{Excess bpc over each schedule's own sequential chain on the 85M model under the protocol
of Table~\ref{tab:predicted}, for three rules that advance at most one coordinate per masked run whenever the runs allow,
beside 2 estimates from kernel reads like those of Table~\ref{tab:betweenkernel} at mask rates
0.1 to 0.7, each step at its own rate. The first-order reference charges a pair of advanced coordinates the kernel at its distance when they share a run and zero when a revealed character lies between them; the estimate with the state between charges every pair the kernel at its distance and
its number of revealed characters between. The third rule advances the midpoints of the longest
runs and spreads its choices among runs of equal length; its estimate was computed before its
decode.}
\label{tab:separation}
\begin{tabular}{llccc}
\toprule
Rule & & 8 steps & 16 steps & 32 steps \\
\midrule
separator & first-order reference & 0.01 & 0.01 & 0.00 \\
 & with the state between & 0.06 & 0.01 & 0.00 \\
 & measured & 0.06 & 0.02 & 0.00 \\
\midrule
midpoint & first-order reference & 0.01 & 0.00 & 0.00 \\
 & with the state between & 0.18 & 0.14 & 0.09 \\
 & measured & 0.26 & 0.19 & 0.09 \\
\midrule
midpoint, ties spread & first-order reference & 0.01 & 0.00 & 0.00 \\
 & with the state between & 0.05 & 0.00 & 0.00 \\
 & measured & 0.07 & 0.01 & 0.00 \\
\bottomrule
\end{tabular}
\end{table}

\subsection{Text8 decoding from 2 to 64 steps}
\label{app:frontier}

The spread and dilated rules are tied in bpc through 8 steps (Figure~\ref{fig:fronts}, Table~\ref{tab:methods}) but diverge from 16 steps onward. The dilated order is fixed before decoding, so its sequential chain is its 64-step reading at every budget, and from 16 steps its bpc lies within 0.01 bits of that chain. The spread rule's chains are lower (Table~\ref{tab:chains}), which gives it the lower bpc although its dependence cost is larger. Table~\ref{tab:frontier} extends these readings to 2 and 4 steps, with standard deviations over sampling seeds.

\begin{table}[htbp]
\centering\small
\caption{Bits per character on the last 64 characters of held-out
windows, teacher forced after a revealed 192-character prefix, using
one weight set over full sequences. Decoding is deterministic except
for the random rule, whose selections vary across the 3 sampling seeds. The 64-step column evaluates the sequential chain.}
\label{tab:frontier}
\setlength{\tabcolsep}{3pt}
\begin{tabular}{lcccccc}
\toprule
Steps & 64 & 32 & 16 & 8 & 4 & 2 \\
\midrule
spread & 1.32 $\pm$ 0.00 & 1.49 $\pm$ 0.00 & 1.56 $\pm$ 0.00 & 1.65 $\pm$ 0.00 & 1.84 $\pm$ 0.00 & 2.31 $\pm$ 0.00 \\
random & 1.58 $\pm$ 0.00 & 1.62 $\pm$ 0.01 & 1.67 $\pm$ 0.01 & 1.80 $\pm$ 0.00 & 2.05 $\pm$ 0.01 & 2.60 $\pm$ 0.01 \\
contiguous & 1.32 $\pm$ 0.00 & 1.77 $\pm$ 0.00 & 2.44 $\pm$ 0.00 & 3.14 $\pm$ 0.00 & 3.66 $\pm$ 0.00 & 3.96 $\pm$ 0.00 \\
\bottomrule
\end{tabular}
\end{table}

\subsection{Orders and kernels on MAR-B}
\label{app:marinfo}

On MAR-B's $16\times16$ token grid a rule chooses only the permutation that MAR reveals, in chunks set by its cosine schedule. The midpoint rule orders the tokens farthest point first: each next token is the one farthest in Chebyshev distance from those already revealed, with ties broken at random. The low-discrepancy rule takes the cells in the order a two-dimensional Halton sequence in bases 2 and 3 first reaches them. The nested rule takes the middle row and column first, then the middle row and column of each part they leave, as in nested dissection \citep{george1973nested}. The spread rule follows the bit-reversed Morton order, and spread then raster takes the first 64 tokens of the spread order and the rest row by row. MAR reveals a permutation from its end, and we reverse the midpoint, low-discrepancy, nested and spread-then-raster orders, which are built first-revealed-first, before MAR reads them.

The kernel of Figure~\ref{fig:kernel}b measures how far a masked token's prediction moves when a
revealed neighbour is perturbed. MAR-B's diffusion head also gives each masked token's law, and the
Kullback--Leibler divergence between that law with and without one more revealed neighbour
measures the information the neighbour carries. Before reading it, we registered that the floors
built from the two kernels would order the rules alike, counting a reversal only where the displacement floors differ by more than 10\%. The information kernel falls much faster with
distance than the displacement kernel (Table~\ref{tab:marinfo}). The two sets of floors order the
rules alike at 8, 16, 32 and 64 steps, except low-discrepancy and nested, whose displacement floors
lie within 3\% of each other. Against the measured FIDs at 8 steps, both place nested well below
random and spread below low-discrepancy, where the measurements put nested above random and
low-discrepancy lowest. The information floor puts midpoint about 15\% below random and the displacement floor about 7\%, against a measured 18\%.

\subsection{Confidence under a minimum distance}

Imposing a minimum distance excludes the closest pairs from
confidence selection. Table~\ref{tab:conf} compares the resulting positions, pairwise estimates and measured bpc. The min-distance rule's pairwise estimate is within 0.01 bits of the spread rule's, yet its bpc remains close to that of the random rule.

\begin{table}[htbp]
\centering\small
\caption{Selection rules at 16 steps over the 64
decoded characters, width 4, using the evaluation protocol of Table~\ref{tab:frontier}. Separation is the mean gap between
consecutive selected positions within a step, 16.0 for an even grid
over the decoded region. The last column is the pairwise estimate of Table~\ref{tab:predicted} on each rule's own sets, in bits per character. Min-distance keeps the positions of a step at least 8 apart.}
\label{tab:conf}
\begin{tabular}{lccc}
\toprule
Selection rule & bpc & separation & pairwise estimate \\
\midrule
confidence & 2.18 & 2.0 & 0.34 \\
min-distance & 1.66 & 11.7 & 0.02 \\
random & 1.67 & 13.0 & 0.06 \\
spread & 1.56 & 16.0 & 0.03 \\
\bottomrule
\end{tabular}
\end{table}

\subsection{Width profiles}
\label{app:families}

For a fixed selection rule and step budget, the width profile
specifies how many coordinates advance at each step. A kernel estimated only as a function of distance assigns the same cost to all contiguous profiles. Every gap between consecutive coordinates selected in a step is 1, and every profile has the same total number of these gaps. Consistent with the distance-only estimate, measured bpc differs by only a few hundredths across contiguous profiles at 8 steps and across the polynomial profiles at 16 steps (Table~\ref{tab:shapes}). The geometric profiles at 16 steps
have bpc up to 0.2 bits higher. Their widest step advances 33\% of the sequence and so includes dependence beyond pairs. Under the spread rule, the uniform profile has the lowest bpc because the wider steps in
geometric profiles reduce the separation between selected coordinates.

\begin{table}[htbp]
\centering\small
\caption{Width profiles compared at matched step counts. We
report bits per character for complete decodes using one weight set,
deterministically evaluated on fixed windows. Every profile
distributes the same total number of coordinate advances over the
same number of steps. Wide-first and wide-late use $w_t
\propto (t{+}1)^{\mp 1}$; the geometric profiles scale width by $2/3$
or $3/2$ per step.}
\label{tab:shapes}
\begin{tabular}{lcccc}
\toprule
& contiguous 8 & contiguous 16 & spread 8 & spread 16 \\
\midrule
uniform & 3.96 & 3.70 & 1.74 & 1.64 \\
wide first & 3.94 & 3.69 & 1.92 & 1.81 \\
wide late & 3.97 & 3.75 & 1.79 & 1.69 \\
geometric, wide first & 3.94 & 3.84 & 1.95 & 1.91 \\
geometric, wide late & 3.99 & 3.89 & 1.86 & 1.85 \\
\bottomrule
\end{tabular}
\end{table}

\subsection{The graded channel}
\label{app:graded}

Proposition~\ref{thm:pathcost} holds for every set of levels $L$. We
define the graded channel as the token channel with one interior level,
at which coordinate $i$ shows a function $b_i = g(x_i)$ of its value.

Let a set $S$ be revealed in two stages, first $b_S$ and then $x_S$. The
chain rule gives
\begin{equation}
\mathrm{TC}(b_S \mid \ell) + \mathrm{TC}(x_S \mid b_S, \ell)
\;=\; \mathrm{TC}(x_S \mid \ell) \;-\; \sum_{i \in S}
\Big[ H(x_i \mid b_i, \ell) - H(x_i \mid b_S, \ell) \Big].
\label{eq:twostage}
\end{equation}
The reduction in cost equals the information that the other
coordinates' interior values provide about $x_i$ beyond its own
interior value $b_i$. Revealing every coordinate in two stages requires
twice as many coordinate advances. At the same step budget, the
steps must therefore be twice as wide. The graded channel reduces
the cost at matched steps when the saving in \eqref{eq:twostage}
exceeds the cost of this increase in width.

The text8 channel uses 3 character classes as its interior level:
space, vowel and consonant. The classes are fixed before training, and
their 3 symbols are added to the vocabulary. Both models draw
every coordinate's level uniformly and independently, the graded model
among mask, class and character and the absorbing model between mask
and character, and score the character of every coordinate not yet
revealed, weighted by the reciprocal of the unrevealed fraction of its
row.

We estimate the bound for each stage from the graded model's kernels
(Table~\ref{tab:gradedkernel}). For both models, the excess bpc over
their own sequential chains exceeds the corresponding pairwise estimate at every budget (Table~\ref{tab:graded}). From 32 steps, the two models' excesses lie within 0.01 bits of each other, and the graded model's higher bpc comes from its sequential chain.

\begin{table}[htbp]
\centering\small
\caption{Sequential-chain bpc, excess bpc from parallel decoding,
and the pairwise estimate for the two models on the schedules in
Table~\ref{tab:partial}. For each model, the chain follows its own
schedule, revealing the coordinates of each set one at a time. The
excess is the difference between the parallel decode and this chain.
The estimate is computed from the model's own kernels
(Table~\ref{tab:gradedkernel}).}
\label{tab:graded}
\begin{tabular}{lcccc}
\toprule
Steps & 8 & 16 & 32 & 64 \\
\midrule
chain, absorbing & 1.586 & 1.564 & 1.545 & 1.542 \\
chain, graded & 1.628 & 1.644 & 1.683 & 1.688 \\
excess, absorbing & 0.105 & 0.036 & 0.039 & 0.012 \\
excess, graded & 0.162 & 0.091 & 0.033 & 0.015 \\
estimate, absorbing & 0.029 & 0.014 & 0.011 & 0.004 \\
estimate, graded (class + character) & 0.057 & 0.013 & 0.008 & 0.007 \\
\bottomrule
\end{tabular}
\end{table}

\begin{table}[htbp]
\centering\small
\caption{Pair kernels of the two models at mask rate 0.7,
following the convention of Figure~\ref{fig:kernel}. Entries give
bits for a pair within one word / a pair across a boundary, with
512 pairs per cell. We report the absorbing model's character kernel
and the graded model's class kernel and character kernel conditioned
on the classes.}
\label{tab:gradedkernel}
\begin{tabular}{lcccc}
\toprule
Distance & 1 & 2 & 4 & 8 \\
\midrule
absorbing, character & 1.172 / 0.797 & 0.900 / 0.446 & 0.609 / 0.315 & 0.374 / 0.112 \\
graded, class & 0.177 / 0.283 & 0.054 / 0.127 & 0.019 / 0.080 & 0.021 / 0.008 \\
graded, character given classes & 0.494 / 0.006 & 0.560 / 0.016 & 0.455 / 0.054 & 0.262 / 0.026 \\
\bottomrule
\end{tabular}
\end{table}

\subsection{Free generation on text8}
\label{app:freegen}

In the teacher-forced evaluation of Section~\ref{sec:frontier}, we
score the true characters at the selected positions and reveal them
before the next step, so sampling errors cannot affect subsequent
conditionals. For free generation, we start from the empty
state and sample the model at the positions selected in each step.
We score the completed text with the separately trained autoregressive
model and measure unigram divergence from the corpus to check whether
a low scorer loss accompanies degenerate text.

The selection rules retain their ranking at every parallel width in
free generation. Rules that select clustered coordinates also have the largest unigram divergence from the corpus (Table~\ref{tab:freegen}). Word-aware spread reads above spread at every budget. At the sequential limit, where the rules differ only in coordinate order, contiguous order gives the lowest scorer loss.

\begin{table}[htbp]
\centering\small
\caption{Free generation at matched steps: bits per character of the
sampled text under the dedicated autoregressive scorer over 3 sampling seeds. The last column is the unigram divergence against the corpus at 16 steps. Word-aware spread is the separator rule with words as the segments, at most one coordinate per word until every word has one, with word boundaries read from the spaces decoded so far and the model's own space predictions.}
\label{tab:freegen}
\begin{tabular}{lccccc}
\toprule
Rule & 8 & 16 & 64 & 256 & unigram KL \\
\midrule
contiguous & 6.35 $\pm$ .01 & 6.11 $\pm$ .01 & 5.38 $\pm$ .03
  & 1.42 $\pm$ .02 & .053 $\pm$ .001 \\
confidence & 6.01 $\pm$ .02 & 5.91 $\pm$ .01 & 4.69 $\pm$ .02
  & 1.82 $\pm$ .01 & .048 $\pm$ .001 \\
random & 3.81 $\pm$ .02 & 3.08 $\pm$ .02 & 2.41 $\pm$ .01
  & 2.17 $\pm$ .01 & .001 $\pm$ .000 \\
spread & 2.60 $\pm$ .01 & 2.16 $\pm$ .02 & 1.98 $\pm$ .02
  & 1.46 $\pm$ .01 & .002 $\pm$ .000 \\
word-aware spread & 2.70 $\pm$ .01 & 2.21 $\pm$ .04 & 2.03 $\pm$ .05
  & 1.80 $\pm$ .01 & .002 $\pm$ .000 \\
\bottomrule
\end{tabular}
\end{table}

\subsection{Generative perplexity on the released language model}
\label{app:lladappl}

We evaluate the selection rules on the released language model using
generative perplexity. At every budget, its kernel gives a larger pairwise estimate for the random rule than for the spread or midpoint rule; from
32 steps onward, the difference is only a few hundredths of a bit per
token. The text8 ranking also predicts higher perplexity for confidence
than for random, and a reduction when confidence selection is
constrained to the minimum distance indicated by the kernel. At 8 and 16 steps, all these
comparisons hold: the 95\% bootstrap intervals for the corresponding log-perplexity differences
exclude zero (Table~\ref{tab:lladappl}). Min-distance has lower perplexity than plain confidence at every budget, and plain confidence has higher perplexity than random at every
budget. At 32 steps, the intervals for all pairwise differences among the random, spread and
midpoint rules include zero. At 64 steps the random rule's perplexity is below the spread rule's,
with an interval that excludes zero, although its pairwise estimate is larger by a few hundredths
of a bit per token.

Repetition lowers generative perplexity, and the min-distance rule's samples repeat more
4-token sequences than plain confidence's at every budget (Table~\ref{tab:lladappl}). We regress each sample's log perplexity on the rule, its unigram entropy and its share of
repeated 4-token sequences. With both measures held fixed, the min-distance rule's log perplexity
stays below plain confidence's at every budget, with bootstrap intervals that exclude zero, and
among the less repetitive half of the samples it keeps most of its reduction
(Table~\ref{tab:lladarep}).

\begin{table}[tb]
\centering\footnotesize
\caption{Change in log generative perplexity (nats) of the min-distance rule against plain
confidence, on the samples of Table~\ref{tab:lladappl}. Pooled: the difference in
Table~\ref{tab:llada}. Controlled: the rule's coefficient when each sample's log perplexity is
regressed on the rule, its unigram entropy and its share of repeated 4-token sequences, with a
95\% bootstrap interval over samples. Less repetitive half: the pooled difference among the samples
whose share of repeated 4-token sequences lies below the median of both rules' samples.}
\label{tab:lladarep}
\begin{tabular}{lcccc}
\toprule
Steps & 8 & 16 & 32 & 64 \\
\midrule
pooled & $-1.04$ & $-0.90$ & $-1.49$ & $-1.56$ \\
controlled & $-0.90$ & $-0.67$ & $-1.20$ & $-0.92$ \\
\quad 95\% interval & {\footnotesize[$-1.08$, $-0.71$]} & {\footnotesize[$-0.88$, $-0.47$]} & {\footnotesize[$-1.55$, $-0.90$]} & {\footnotesize[$-1.25$, $-0.59$]} \\
less repetitive half & $-0.83$ & $-0.61$ & $-1.16$ & $-1.48$ \\
\bottomrule
\end{tabular}
\end{table}

\begin{table}[tb]
\centering\scriptsize
\setlength{\tabcolsep}{3.5pt}
\caption{Generative perplexity on released LLaDA-8B-Base weights,
with only the selection rule changed in the sampler. We generate 512-token samples with nucleus
sampling at probability threshold 0.9; confidence ranks positions by the probability of the
predicted token, and GPT-2-large scores each whole sample, re-tokenised, with a sliding window.
Each cell pools 192 samples over 3 sampling seeds and reports perplexity with a 95\%
bootstrap interval over samples. Min-distance keeps the positions of a step at least 6 apart, the first distance at which the kernel on these weights falls to 10\% of its value at distance 1; when it leaves too few positions, the most confident
remaining positions fill the step. The lower blocks report the samples' unigram entropy in bits
and their share of repeated 4-token sequences in the scorer's tokens, averaged over the 192
samples.}
\label{tab:lladappl}
\begin{tabular}{@{}lcccc@{}}
\toprule
Steps & 8 & 16 & 32 & 64 \\
\midrule
random & 283 [246,323] & 147 [123,175] & 54 [47,61] & 27 [23,32] \\
confidence & 375 [331,431] & 212 [182,245] & 180 [142,239] & 97 [74,130] \\
min-distance & 133 [114,157] & 86 [71,105] & 41 [33,52] & 20 [17,24] \\
midpoint & 163 [147,182] & 94 [83,106] & 50 [44,57] & 31 [27,35] \\
spread & 168 [149,190] & 112 [98,130] & 54 [47,64] & 35 [31,41] \\
\midrule
\multicolumn{5}{@{}l}{unigram entropy, bits} \\
random & 7.70 & 7.46 & 7.16 & 7.09 \\
confidence & 7.68 & 7.49 & 7.02 & 6.16 \\
min-distance & 7.55 & 7.22 & 6.72 & 6.56 \\
midpoint & 7.64 & 7.46 & 7.31 & 7.04 \\
spread & 7.64 & 7.46 & 7.30 & 7.05 \\
\midrule
\multicolumn{5}{@{}l}{repeated 4-token sequences, \%} \\
random & 5.0 & 6.9 & 10.1 & 13.4 \\
confidence & 3.6 & 3.9 & 5.1 & 11.0 \\
min-distance & 7.3 & 10.5 & 15.6 & 22.4 \\
midpoint & 7.0 & 7.5 & 9.1 & 13.7 \\
spread & 6.1 & 6.4 & 9.8 & 13.2 \\
\bottomrule
\end{tabular}
\end{table}

\subsection{The cost of selecting}
\label{app:overhead}

Each selection rule computes positions or sorts the model's output.
On both decoders, selection takes less than 1\% of the time
for a forward pass (Table~\ref{tab:overhead}). The minimum-distance
constraint takes the longest, below 0.75\% of the
forward pass on the released model and below 0.5\% on the
text8 decoder.

\begin{table}[htbp]
\centering\small
\caption{Forward-pass and selection time per decoding step,
reported as medians over the timed steps on one L40. The device is
synchronised around each timed region. We evaluate LLaDA-8B-Base at
64 steps on the first 20 GSM8K problems in Table~\ref{tab:llada}, one
prompt at a time, and the text8 85M decoder at 16 steps on 256 windows with a revealed
192-character prefix.}
\label{tab:overhead}
\setlength{\tabcolsep}{4pt}
\begin{tabular}{lcccccc}
\toprule
& \multicolumn{3}{c}{LLaDA-8B-Base, 64 steps} & \multicolumn{3}{c}{text8 85M, 16 steps} \\
\cmidrule(lr){2-4}\cmidrule(lr){5-7}
Rule & forward, ms & select, ms & \% & forward, ms & select, ms & \% \\
\midrule
contiguous & --- & --- & --- & 211.3 & 0.09 & 0.04 \\
random & 170.8 & 0.12 & 0.07 & 212.3 & 0.13 & 0.06 \\
confidence & 170.9 & 0.28 & 0.17 & 216.9 & 0.22 & 0.10 \\
min-distance & 170.3 & 1.24 & 0.73 & 217.6 & 1.01 & 0.46 \\
separator constraint & 171.6 & 0.78 & 0.46 & --- & --- & --- \\
separator & --- & --- & --- & 214.4 & 0.97 & 0.45 \\
midpoint & 171.7 & 0.73 & 0.42 & 215.4 & 0.70 & 0.33 \\
spread & 170.3 & 0.92 & 0.54 & 213.0 & 0.63 & 0.30 \\
\bottomrule
\end{tabular}
\end{table}

\subsection{HumanEval on the released language model}
\label{app:humaneval}

We evaluate confidence and its two constrained forms on HumanEval
under the same sampler (Table~\ref{tab:humaneval}). Here and in Table~\ref{tab:llada} every prompt
is decoded on its own, since decoding prompts in left-padded batches changes the model's outputs. We compare paired outcomes on the 164 problems using McNemar's exact test. Both constrained
rules have higher pass@1 than plain confidence from 8 to 64 steps, with significant paired
differences at each of these budgets, and the two constrained rules tie at every budget. At 128 steps, where each step advances 2 coordinates, the three rules tie.
Under LLaDA's sampling at 64 steps on GSM8K, the min-distance rule has higher accuracy than plain confidence, paired over the 1,319 problems (Table~\ref{tab:humaneval}), where greedy decoding at the same budget gives a tie (Table~\ref{tab:llada}).

\begin{table}[htbp]
\centering\footnotesize\setlength{\tabcolsep}{3pt}
\caption{HumanEval pass@1 in percent on released LLaDA-8B-Base
weights, with only the selection rule changed in LLaDA's sampler.
We use greedy decoding on all 164 problems, one prompt at a time, score
completions with the original code-execution harness in a sandbox and
report binomial standard errors. Min-distance keeps the positions of a step at least 6 apart. The lower block reports GSM8K strict match at 64 steps on the 1,319
problems, using LLaDA's sampling procedure in place of greedy decoding; the paired difference is $4.5 \pm 1.6$ points. Under plain confidence at 256 steps, the same harness gives 71.4 $\pm$ 1.2 on GSM8K, against the 70.0 that LLaDA's repository reports at this setting.}
\label{tab:humaneval}
\begin{tabular}{lccccc}
\toprule
Steps & 8 & 16 & 32 & 64 & 128 \\
\midrule
confidence & 1.8 $\pm$ 1.0 & 2.4 $\pm$ 1.2 & 5.5 $\pm$ 1.8 & 11.0 $\pm$ 2.4 & 26.2 $\pm$ 3.4 \\
min-distance & 7.9 $\pm$ 2.1 & 9.1 $\pm$ 2.3 & 12.2 $\pm$ 2.6 & 17.7 $\pm$ 3.0 & 26.8 $\pm$ 3.5 \\
separator constraint & 7.9 $\pm$ 2.1 & 8.5 $\pm$ 2.2 & 13.4 $\pm$ 2.7 & 18.3 $\pm$ 3.0 & 25.6 $\pm$ 3.4 \\
\midrule
\multicolumn{6}{@{}l}{GSM8K, sampling, 64 steps} \\
confidence & \multicolumn{5}{c}{41.3 $\pm$ 1.4} \\
min-distance & \multicolumn{5}{c}{45.8 $\pm$ 1.4} \\
\bottomrule
\end{tabular}
\end{table}

\section{The Frame Axis on Released Video Weights}
\label{sec:released}

\begin{table}[!ht]
\centering\footnotesize\setlength{\tabcolsep}{3pt}
\caption{Memory depth and decoding steps on the released SkyReels-V2 checkpoint \citep{skyreels2025} at 8 times its training horizon, in the 4 VBench dimensions declared before decoding, measured over the entire 773-frame
rollout with every frame denoised. The shipped row gives the mean over 3 prompts; the other rows give the paired
difference from it over the same prompts, with its paired standard error.}
\label{tab:videoall}
\begin{tabular}{lccccc}
\toprule
 & & subject & background & motion & imaging \\
Memory depth & steps & consistency & consistency & smoothness & quality \\
\midrule
17 frames, as shipped & 300 & 0.8998 & 0.9505 & 0.9937 & 65.70 \\
5 frames & 270 & $-0.0087 \pm 0.0206$ & $-0.0053 \pm 0.0047$ & $-0.0015 \pm 0.0013$ & $-3.49 \pm 2.44$ \\
25 frames & 330 & $-0.0068 \pm 0.0130$ & $-0.0033 \pm 0.0120$ & $+0.0000 \pm 0.0004$ & $-3.50 \pm 1.43$ \\
33 frames & 360 & $-0.0109 \pm 0.0128$ & $-0.0007 \pm 0.0062$ & $+0.0001 \pm 0.0006$ & $-0.34 \pm 1.41$ \\
57 frames & 540 & $-0.0278 \pm 0.0181$ & $-0.0190 \pm 0.0026$ & $+0.0009 \pm 0.0005$ & $-5.15 \pm 3.15$ \\
\bottomrule
\end{tabular}
\end{table}

We also compare schedules on released video weights using the kernel.
A checkpoint from the Diffusion Forcing family supports an independent
level per frame, allowing schedules to vary the memory depth: the
number of earlier frames on which each segment conditions
\citep{chen2024diffusionforcing, skyreels2025}. Before decoding, we
estimate the kernel between frames on these weights. It falls to about 60\% of its value
between adjacent frames within 6 latent frames and changes little beyond, whereas the token kernel in Figure~\ref{fig:kernel} falls to near zero by distance 16. The plateau is dependence that every earlier frame shares with the segment, which the nearest frames already carry. The pairwise estimate therefore predicts that greater memory depth keeps quality the same while requiring more steps. Table~\ref{tab:videoall} reports the 4 VBench dimensions we declared before decoding. In subject consistency, every other setting lies below the shipped one, within 2 paired standard errors. We decode with the synchronised schedule and causal blocks of 1 frame, so that each segment conditions on
exactly the frames its row states, and every frame is denoised before the rollout is scored. With
blocks of 5 frames the released code truncates each segment's prefix to whole blocks, and the
5-, 25- and 33-frame depths would condition on 0, 17 and 17 frames. We
compute the step count directly from the schedule and score the entire rollout so that the
evaluation includes discontinuities between segments, which separate scoring of each segment
would miss. On every dimension, every setting scores at most 2 paired standard errors above the shipped one, and 2 cells score more than 2 below it: background consistency at 57
frames and imaging quality at 25 frames.

\end{document}